\documentclass[11pt]{article}
\usepackage{amsmath}
\usepackage{amsthm}
\usepackage{amssymb}
\usepackage{algorithm}
\usepackage{subfig}
\usepackage{color}
\usepackage[english]{babel}
\usepackage{graphicx}
\usepackage{grffile}
\usepackage{wrapfig,epsfig}
\usepackage{epstopdf}
\usepackage{url}
\usepackage{color}
\usepackage{epstopdf}
\usepackage{algpseudocode}
\usepackage[T1]{fontenc}
\usepackage{bbm}
\usepackage{comment}

\let\C\relax
\usepackage{tikz}
\usepackage{hyperref}  
\hypersetup{colorlinks=true,citecolor=blue,linkcolor=blue} 
\usetikzlibrary{arrows}
\usepackage[margin=1in]{geometry}
\graphicspath{{./figs/}}

\newcommand{\fre}[1]{\cos\left(\frac{2\pi f_{#1}t}{T} + \theta_{#1}\right)}

\newtheorem{theorem}{Theorem}[section]
\newtheorem{lemma}[theorem]{Lemma}

\newtheorem{corollary}[theorem]{Corollary}

\newtheorem{remark}[theorem]{Remark}
\newtheorem{claim}[theorem]{Claim}

\newcommand{\wh}{\widehat}

\newcommand{\ov}{\overline}

\newcommand{\R}{\mathbb{R}}

\renewcommand{\i}{\mathbf{i}}

\renewcommand{\varepsilon}{\epsilon}

\renewcommand{\hat}{\wh}
\renewcommand{\R}{\mathbb{R}}

\DeclareMathOperator*{\C}{\mathbb{C}}

\DeclareMathOperator{\poly}{poly}

\DeclareMathOperator{\diag}{diag}

\makeatletter
\newcommand*{\RN}[1]{\expandafter\@slowromancap\romannumeral #1@}
\makeatother

\usepackage{lineno}

\usepackage{etoolbox}

\makeatletter

\newcommand{\define}[4][ignore]{%
  \ifstrequal{#1}{ignore}{}{
  \@namedef{thmtitle@#2}{#1}}%
  \@namedef{thm@#2}{#4}%
  \@namedef{thmtypen@#2}{lemma}%
  \newtheorem{thmtype@#2}[theorem]{#3}%
  \newtheorem*{thmtypealt@#2}{#3~\ref{#2}}%
}

\newcommand{\state}[1]{%
  \@namedef{curthm}{#1}
  \@ifundefined{thmtitle@#1}{
  \begin{thmtype@#1}
    }{
  \begin{thmtype@#1}[\@nameuse{thmtitle@#1}]
  }
    \label{#1}
    \@nameuse{thm@#1}
  \end{thmtype@#1}
  \@ifundefined{thmdone@#1}{
  \@namedef{thmdone@#1}{stated}%
  }{}
}

\newcommand{\restate}[1]{%
  \@namedef{curthm}{#1}
  \@ifundefined{thmtitle@#1}{
    \begin{thmtypealt@#1}
    }{
  \begin{thmtypealt@#1}[\@nameuse{thmtitle@#1}]
  }
    \@nameuse{thm@#1}
  \end{thmtypealt@#1}
  \@ifundefined{thmdone@#1}{
  \@namedef{thmdone@#1}{stated}%
  }{}
}

\newcommand{\thmlabel}[1]{
  \@ifundefined{thmdone@\@nameuse{curthm}}{\label{#1}
    }{\tag*{\eqref{#1}}}
}
\makeatother

\newcommand{\DefMacro}[2]{\expandafter\newcommand\csname rmk-#1\endcsname{#2}}
\newcommand{\UseMacro}[1]{\csname rmk-#1\endcsname}

\usepackage[nomessages]{fp}

\def\DeltaScale#1#2#3{%
    \FPeval\result{clip(((#1)-(#2))*(#3))}%
\relax{\result}%
}

\begin{document}
\date{}
\title{Learning Long Term Dependencies via Fourier Recurrent Units}
\author{
  Jiong Zhang\\
  \texttt{zhangjiong724@utexas.edu}\\
  UT-Austin
  \and
  Yibo Lin\\
  \texttt{yibolin@utexas.edu}\\
  UT-Austin
  \and
  Zhao Song\thanks{Part of the work was done while hosted by Jelani Nelson.}\\
  \texttt{zhaos@g.harvard.edu}\\
  Harvard University \& UT-Austin
  \and
  Inderjit S. Dhillon\\
  \texttt{inderjit@cs.utexas.edu}\\
  UT-Austin
}


\begin{titlepage}
 \maketitle
  \begin{abstract}
It is a known fact that training recurrent neural networks for tasks 
that have long term dependencies is challenging. One of the main reasons is the vanishing or 
exploding gradient problem, which prevents gradient information from 
propagating to early 
layers. In this paper we propose a simple recurrent architecture, the Fourier Recurrent 
Unit (FRU), that stabilizes the gradients that arise in its training while giving 
us stronger expressive power. Specifically, FRU summarizes 
the hidden states $h^{(t)}$ along the temporal dimension with Fourier basis functions. This 
allows gradients to easily reach any layer due to FRU's residual learning 
structure and the global support of 
trigonometric functions. We show that FRU has gradient lower and upper bounds 
independent of temporal dimension. We also show the strong expressivity of 
sparse Fourier basis, from which FRU obtains its strong expressive power. Our 
experimental study also demonstrates that with fewer parameters the proposed architecture 
outperforms other recurrent architectures on many tasks.

  \end{abstract}
 \thispagestyle{empty}
 \end{titlepage}

\section{Introduction}\label{sec:intro}
Deep neural networks~(DNNs) have shown remarkably better 
performance than classical models on a wide range of problems, including 
speech recognition, computer vision and natural language processing. 
Despite DNNs having tremendous expressive power to fit 
very complex functions, training them by back-propagation
can be difficult. Two main issues are vanishing and exploding gradients. 
These issues become particularly troublesome for recurrent neural 
networks~(RNNs) since the weight matrix is identical at each layer and
any small changes get amplified exponentially through the recurrent 
layers~\cite{bengio1994learning}. 
Although exploding gradients can be somehow mitigated by tricks like gradient 
clipping or normalization~\cite{pascanu2013difficulty}, vanishing gradients are 
harder to deal with. If gradients vanish, there is little information 
propagated back through back-propagation. This means that deep RNNs have great difficulty 
learning long-term dependencies. 

Many models have been proposed to address the vanishing/exploding gradient 
issue for DNNs. For example Long Short Term 
Memory~(LSTM)~\cite{hochreiter1997long} tries to solve it by adding additional 
memory gates, while residual networks~\cite{he2016deep} add a short cut to 
skip intermediate layers. Recently the approach of directly obtaining the 
statistical summary of past layers has drawn attention, such as statistical 
recurrent units~(SRU)~\cite{oliva2017statistical}. However, as we show later, they still suffer 
from vanishing gradients and have limited access to past layers.

In this paper, we present a novel recurrent architecture, 
Fourier Recurrent Units~(FRU) that use Fourier basis 
to summarize the hidden statistics over past time steps. We show that this solves the 
vanishing gradient problem and gives us access to any past time step region. In 
more detail, we make the following contributions: 
\begin{itemize}
    \item We propose a method to summarize hidden states through past time 
    steps in a recurrent neural network with Fourier basis~(FRU). Thus any 
      statistical summary of past hidden states can be approximated by a linear 
      combination of summarized Fourier statistics.
    \item Theoretically, we show the expressive power of sparse Fourier basis 
      and prove that FRU can solve the vanishing gradient problem by looking 
      at gradient norm bounds. Specifically, we show that in the linear setting, SRU only
      improves the gradient lower/upper bound of RNN by a constant factor of the 
      exponent~(i.e, both have the form $(e^{a T},e^{b T})$), while FRU (lower and upper) bounds the gradient by constants 
      independent of the temporal dimension.
    \item We tested FRU together with RNN, LSTM and SRU on both synthetic and 
    real world datasets like pixel-(permuted) MNIST, IMDB movie rating dataset. FRU shows its 
    superiority on all of these tasks while enjoying smaller number of 
    parameters than LSTM/SRU. 
\end{itemize}

We now present the outline of this paper. In Section~\ref{sec:related_work} we discuss 
related work, while in Section~\ref{sec:motiv} we introduce the FRU 
architecture and explain the intuition regarding the statistical summary and 
residual learning. In Sections~\ref{sec:expower} and~\ref{sec:vangrad}
we prove the expressive power of sparse Fourier basis and show that in the linear 
case FRUs have constant lower and upper bounds on gradient magnitude.
Experimental results on benchmarking synthetic datasets as well as real datasets 
like pixel MNIST and language data are presented in 
Section~\ref{sec:exp}. Finally, we present our conclusions and suggest several 
interesting directions in Section~\ref{sec:conclu}. 

\section{Related Work}\label{sec:related_work}
Numerous studies have been conducted hoping to address the vanishing and 
exploding gradient problems, such as the use of self-loops and gating units in the 
LSTM~\cite{hochreiter1997long} and GRU~\cite{cho2014properties}. These models use
trained gate units on inputs or memory states to keep the memory for a longer 
period of time thus enabling them to capture longer term dependencies than RNNs.
However, it has also been argued that by using a simple initialization trick, RNNs 
can have better performance than LSTM on some benchmarking tasks~\cite{le2015simple}.
Apart from these advanced frameworks, straight forward methods like gradient 
clipping~\cite{mikolov2012statistical} and spectral 
regularization~\cite{pascanu2013difficulty} are also proposed. 

As brought to wide notice in Residual networks~\cite{he2016deep}, give MLP 
and CNN shortcuts to skip intermediate layers allowing gradients to flow back and 
reach the first layer without being diminished. It is also claimed this helps to
preserve features that are already good. Although ResNet is 
originally developed for MLP and CNN architectures, many extensions to RNN 
have shown improvement, such as maximum entropy 
RNN~(ME-RNN)~\cite{mikolov2011strategies}, highway 
LSTM~\cite{zhang2016highway} and Residual LSTM~\cite{kim2017residual}.

Another recently proposed method, the statistical recurrent unit~(SRU)~\cite{oliva2017statistical}, 
keeps moving averages of summary statistics through past time steps. 
Rather than use gated units to decide what should be memorized, at each layer 
SRU memory cells incorporate new information at rate $\alpha$ and forget old 
information by rate $(1-\alpha)$. Thus by linearly combining 
multiple memory cells with different $\alpha$'s, SRU can have a multi-scale 
view of the past. However, 
the weight of moving averages is exponentially decaying through time and will 
surely go to zero given enough time steps. This prevents SRU from accessing the 
hidden states a few time steps ago, and allows gradients to vanish. Also, the 
expressive power of the basis of exponential functions is small which 
limits the expressivity of the whole network. 

Fourier transform is a strong mathematical tool that has been successful in 
many applications. However the previous studies of Fourier expressive power 
have been concentrate in dense Fourier transform. Price and Song \cite{ps15} 
proposed a way to define $k$-sparse Fourier transform problem in the continuous 
setting and also provided an algorithm which requires the frequency gap. Based 
on that \cite{ckps16} proposed a frequency gap free algorithm and well defined 
the expressive power of $k$-sparse Fourier transform. One of the key 
observations in the frequency gap free algorithm is that a low-degree polynomial has 
similar behavior as Fourier-sparse signal. To understand the expressive power 
of Fourier basis, we use the framework designed by \cite{ps15} and use the 
techniques from \cite{ps15,ckps16}. 

There have been attempts to combine the Fourier transform with RNNs:
the Fourier RNN~\cite{koplon1997using} uses $e^{ix}$ as activation function in 
RNN model;
ForeNet~\cite{zhang2000forenet} notices the similarity between 
Fourier analysis of time series and RNN predictions and arrives at an RNN with 
diagonal transition matrix. For CNN, the FCNN~\cite{pratt2017fcnn} replaces sliding window approach
with the Fourier transform in the convolutional layer. Although some of these 
methods show improvement over current ones, they have not fully exploit the 
expressive power of Fourier transform or avoided the gradient 
vanishing/exploding issue. Motivated by the shortcomings of the above methods, we have developed a method 
that has a thorough view of the past hidden states, has strong 
expressive power and does not suffer from the gradient vanishing/exploding problem.

{\bf Notation.} We use $[n]$ to denote $\{1,2,\cdots,n\}$. We provide several 
definitions related to matrix $A$. Let $\det(A)$ denote the determinant of a 
square matrix $A$, and $A^\top$ denote the transpose of $A$. Let $\| A\|$ denote 
the spectral norm of matrix $A$, and let $A^t$ denote the square matrix $A$ multiplied 
by itself $t-1$ times. Let $\sigma_i(A)$ denote the $i$-th largest singular 
value of $A$. For any function $f$, we define $\widetilde{O}(f)$ to be 
$f\cdot \log^{O(1)}(f)$. In addition to $O(\cdot)$ notation, for two functions 
$f,g$, we use the shorthand $f\lesssim g$ (resp. $\gtrsim$) to indicate that 
$f\leq C g$ (resp. $\geq$) for an absolute constant $C$. We use $f\eqsim g$ to 
mean $cf\leq g\leq Cf$ for constants $c$ and $C$.
Appendix provides the detailed proofs and additional experimental results for comparison.

\section{Fourier Recurrent Unit}\label{sec:motiv}
In this section, we first introduce our notation in the RNN framework and then describe our 
method, the Fourier Recurrent Unit~(FRU), in detail. Given a hidden state 
vector from the previous time step $h^{(t-1)}\in\R^{n_h}$, input 
$x^{(t-1)}\in \R^{n_i}$, RNN computes the next hidden state $h^{(t)}$ and 
output $y^{(t)}\in\R^{n_y}$ as:
\begin{align}\label{eq:rnn_update_rule}
    h^{(t)} &= \phi(W \cdot h^{(t-1)}+U \cdot x^{(t-1)}+b) &\in \R^{n_h}\\
    y^{(t)} &= Y \cdot h^{(t)} &\in\R^{n_y}\notag
\end{align}
where $\phi$ is the activation, $W\in\R^{n_h\times n_h}$, 
$U\in\R^{n_h\times n_i}$ and $Y\in\R^{n_y\times n_h}$, $t=1,2,\ldots,T$ is the time step and $h^{(t)}$ 
is the hidden state at 
step $t$. In RNN, the output $y^{(t)}$ at each step is locally dependent to 
$h^{(t)}$ and only remotely linked with previous hidden states (through multiple 
weight matrices and activations). This give rise to the idea of directly 
summarizing hidden states through time. 

\paragraph{Statistical Recurrent Unit.}
For each $t \in \{1,2,\cdots,T\}$, \cite{oliva2017statistical} propose SRU 
with the following update rules 
\begin{align}\label{eq:sru_update_rule}
  g^{(t)} = & ~ \phi( W_1 \cdot u^{(t-1)} + b_1 ) &\in \R^{n_g} \notag \\
  h^{(t)} = & ~ \phi(W_2 \cdot g^{(t)} + U \cdot x^{(t-1)} + b_2 ) &\in \R^{n_h} \notag \\
  u^{(t)}_i = & ~ D \cdot u^{(t-1)}_i + (I - D)\cdot(\mathbf{1}\otimes I) \cdot h^{(t)}\\
  y^{(t)} = & ~ Y \cdot u^{(t)} & \in \R^{n_y} \notag 
\end{align}
where $\mathbf{1}\otimes I=[I_{n_h},\ldots,I_{n_h}]^\top$.
Given the decay factors $\alpha_k \in (0,1),k=1,2\cdots K$, the 
decaying matrix $D\in\R^{Kn_h\times Kn_h}$ is:
$$  
D = \diag\left(\alpha_1 I_{n_h}, \alpha_2 I_{n_h},\cdots,\alpha_K I_{n_h}\right).
$$
For each $i\in [Kn_h]$ and 
$t>0$, $u^{(t)}_i$ can be expressed as the summary statistics across previous 
time steps with the corresponding $\alpha_k$: 
\begin{align}\label{eq:sru_summary}
  u^{(t)}_{i} = \alpha_k^{t} \cdot u^{(0)}_{i} + (1-\alpha_k) \sum_{\tau=1}^{t} \alpha_k^{t-\tau} \cdot h^{(\tau)}.
\end{align}
However, it is easy to note from \eqref{eq:sru_summary} that the weight on 
$h^{(\tau)}$ vanishes 
exponentially with $t-\tau$, thus the SRU cannot access hidden states from a 
few time steps ago. As we show later in 
section~\ref{sec:vangrad}, the statistical factor only improves the gradient lower 
bound by a constant factor on the exponent and still suffers from vanishing 
gradient. Also, the span of exponential functions has limited 
expressive power and thus linear combination of entries of $u^{(t)}$ also have 
limited expressive power.

\paragraph{Fourier Recurrent Unit.}
Recall that Fourier expansion indicates that a continuous function 
$F(t)$ defined on $[0,T]$ can be expressed as:{
\begin{align*}
  F(t)= & ~ A_0 + \frac{1}{T}\sum_{k=1}^N A_k \cos\left(\frac{2\pi k t}{T} + \theta_k \right)
\end{align*}}
where $\forall k\in[N]$: {
\begin{align*}
  A_k = & ~ \sqrt{a_k^2 + b_k^2}, \theta_k = \text{arctan}(b_k,a_k)\\
  a_k = & ~ 2 \langle F(t), \cos\left( \frac{2\pi k t}{T} \right) \rangle, b_k = ~ 2 \langle F(t), \sin\left( \frac{2\pi k t}{T} \right) \rangle,
\end{align*}}
where $\langle a,b\rangle=\int_0^T a(t)b(t)dt$.
To utilize the strong 
expressive power of Fourier basis, we propose the Fourier recurrent unit model.
Let $f_1, f_2, \cdots, f_K$ denote a set of $K$ frequencies.
For each $t \in \{1,2,\cdots,T\}$, we have the following update rules
\begin{align}\label{eq:fru_update_rule}
  g^{(t)} = & ~ \phi( W_1 \cdot u^{(t-1)} + b_1 ) &\in \R^{n_g} \notag \\
  h^{(t)} = & ~ \phi(W_2 \cdot g^{(t)} + U \cdot x^{(t-1)} + b_2 ) &\in \R^{n_h} \notag \\
  u^{(t)} = & ~ u^{(t-1)} + \frac{1}{T} C^{(t)} \cdot h^{(t)} &\in \R^{n_u} \\
  y^{(t)} = & ~ Y \cdot u^{(t)} & \in \R^{n_y} \notag
\end{align}
where $C^{(t)}\in\R^{n_u\times n_h}$ is the Cosine matrix containing $m$ 
square matrices: {
\begin{align*}
  C^{(t)}
  =
\begin{bmatrix}
  C^{(t)}_1  &
  C^{(t)}_2  &
    \cdots   &
  C^{(t)}_M  
\end{bmatrix}^\top,
\end{align*}}
and each $C^{(t)}_j$ is a diagonal matrix with cosine at $m=\frac{K}{M}$ 
distinct frequencies evaluated at time step $t$: {
\begin{align*}
  C^{(t)}_j
  = \diag\left( \fre{k_1} I_d, \cdots, \fre{k_2} I_d\right)
\end{align*}}
where $k_1 = m(j-1)+1$, $k_2 = mj$ and $d$ is the dimension for each 
frequency. For every $t,j,k>0$, $i=d(k-1)+j$ 
the entry $u^{(t)}_i$ has the expression:
\begin{align}\label{eq:fru_summary}
u^{(t)}_i = u^{(0)}_i + \frac{1}{T} \sum_{\tau=1}^t \cos\left(\frac{2\pi f_k \tau}{T} + \theta_k\right) \cdot h^{(\tau)}_j
\end{align}

\begin{figure}[t]
  \includegraphics[width=\linewidth]{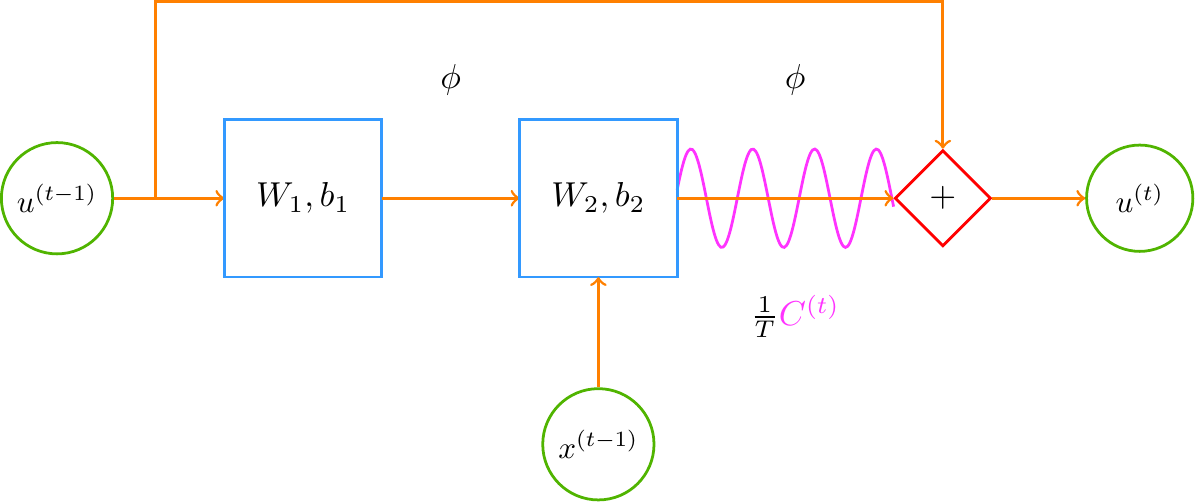}
  \caption{The Fourier Recurrent Unit}
  \label{fig:fru}
\end{figure}

As seen from \eqref{eq:fru_summary}, due to the  
global support of trigonometric functions, we can directly link $u^{(t)}$ with 
hidden states at any time step. Furthermore, because of the expressive power of 
the Fourier basis, given enough frequencies, $y^{(t)}=Y\cdot u^{(t)}$ can 
express any summary statistic of previous hidden states. As we will prove in 
later sections, these features 
prevent FRU from vanishing/exploding gradients and give it much stronger expressive 
power than RNN and SRU. 

\paragraph{Connection with residual learning.}

Fourier recurrent update of $u^{(t)}$ can also be written as:
\begin{align}
  u^{(t+1)} = & ~ u^{(t)} + \mathcal{F}(u^{(t)}) \notag\\
  \mathcal{F}(u^{(t)}) = & ~ \frac{1}{T} C^{(t+1)} \phi(W_2 \phi( W_1 u^{(t)} + b_1 ) + U x^{(t)} + b_2 ) \notag
\end{align}
Thus the information flows from layer $(t-1)$ to layer $t$ along two paths. The second term, $u^{(t-1)}$ needs to pass two 
layers of non-linearity, several weight matrices and scaled down by $T$, while 
the first term, $u^{(t-1)}$ directly goes to $u^{(t)}$ with only identity 
mapping. 
Thus FRU directly incorporates the idea of residual learning while limiting the 
magnitude of the residual term. This not only helps the 
information to flow more smoothly along the temporal dimension, but also acts 
as a regularization that makes the gradient of adjacent layers to be close to identity:
\begin{align}
  \frac{\partial u^{(t+1)}}{\partial u^{(t)}} = & ~ I + \frac{\partial 
  \mathcal{F}}{\partial u^{(t)}}. \notag
\end{align}
Intuitively this solves the gradient exploding/vanishing issue. 
Later in Section~\ref{sec:vangrad}, we give a formal proof and comparison with SRU/RNN. 

\section{Fourier Basis}\label{sec:expower}
\begin{figure*}[t]
	\begin{center}
    {\includegraphics[width=.9\textwidth]{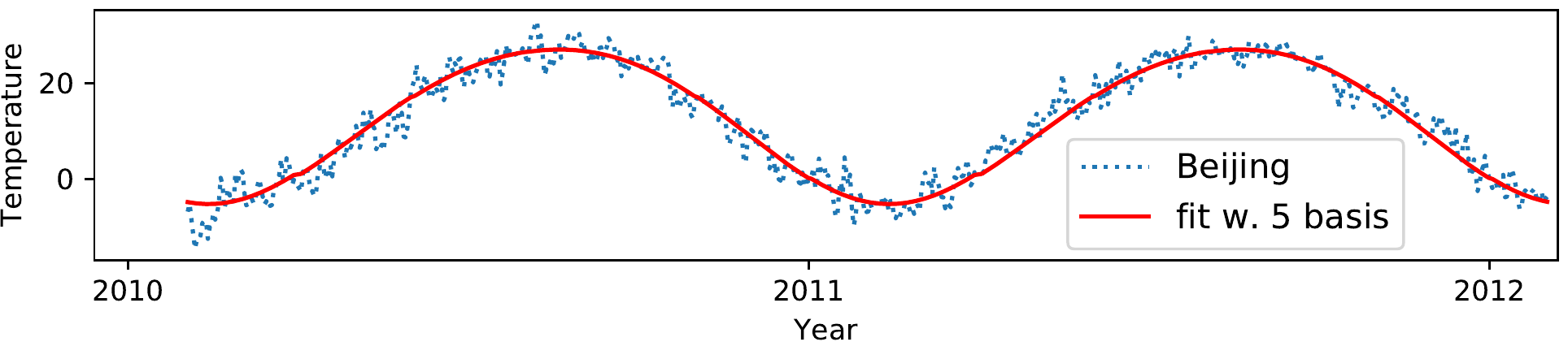}}\label{yearly_Beijing5} \\
    
    {\includegraphics[width=.9\textwidth]{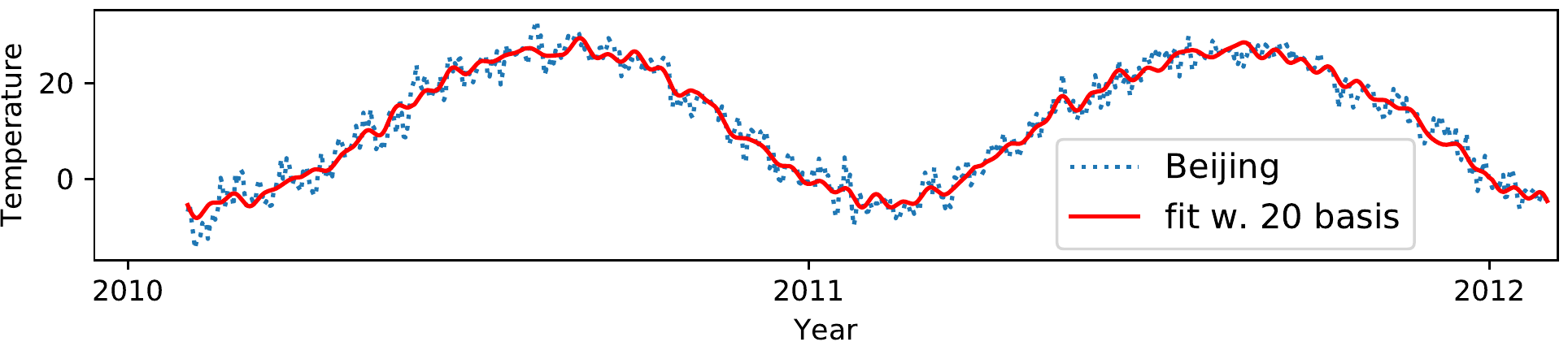}}\label{yearly_Beijing20}\\
    
    {\includegraphics[width=.9\textwidth]{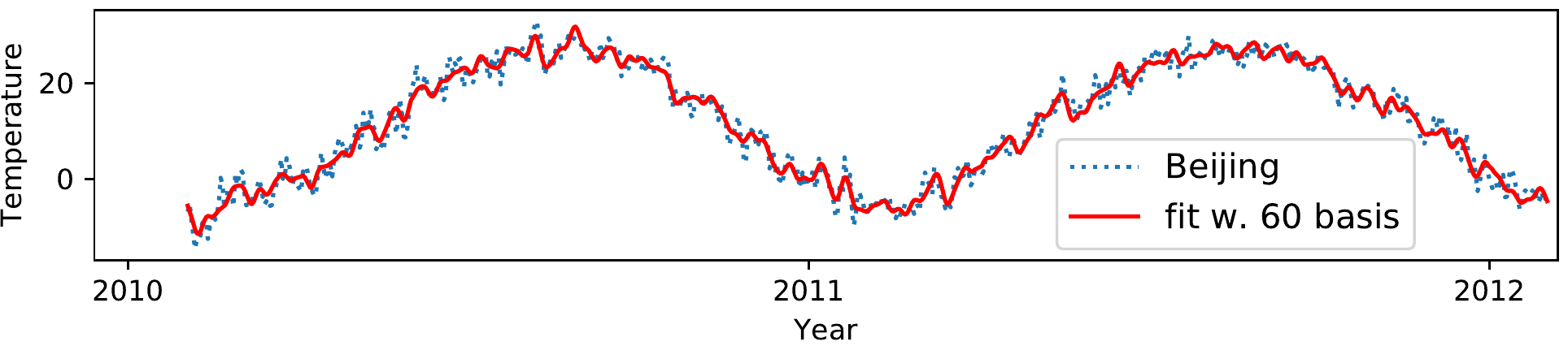}}\label{yearly_Beijing60} \\
   
    {\includegraphics[width=.9\textwidth]{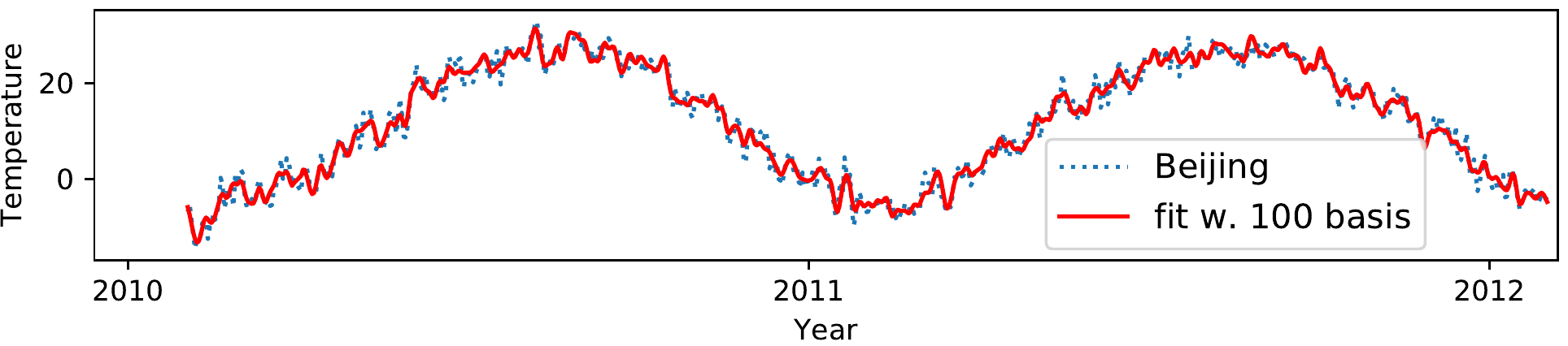}}\label{yearly_Beijing100}
    
    \caption{~Temperature changes of Beijing from year 2010 to 2012, and the fit with Fourier basis: (a) 5 Fourier basis; (b) 20 Fourier basis; (c) 60 Fourier basis; (d) 100 Fourier basis. }
    \end{center}
\end{figure*}
In this section we show that FRU has stronger expressive power than SRU by 
comparing the expressive power of limited number of Fourier basis (sparse 
Fourier basis) and exponential functions. 
On the one hand, we show that sparse Fourier basis is able to approximate 
polynomials well. On the other hand, we prove that even infinitely many exponential 
functions cannot fit a constant degree polynomial. 

First, we state several basic facts which will be later used in the proof.
\begin{lemma}\label{lem:vandermonde}
Given a square Vandermonde matrix $V$ where $V_{i,j} = \alpha_{i}^{j-1}$, then
$
\det(V) = \prod_{1\leq i < j \leq n} (\alpha_j - \alpha_i).
$
\end{lemma}
Also recall the Taylor expansion of $\sin (x)$ and $\cos (x)$ is
\begin{align*}
\sin (x) = & ~ \sum_{i=0}^{\infty} \frac{ (-1)^i }{ (2i+1) ! } x^{2i+1}, \cos (x) = \sum_{i=0}^{\infty} \frac{ (-1)^i }{ (2i) ! } x^{2i}.
\end{align*}

\subsection{Using Fourier Basis to Interpolate Polynomials}
\cite{ckps16} proved an interpolating result which uses Fourier basis (
$e^{2\pi \i f t}$, $\i =\sqrt{-1}$) to fit a complex polynomial ($Q(t): \R \rightarrow \C$). 
However 
in our application, the target polynomial is over 
the real domain, i.e. $Q(t): \R \rightarrow \R$. Thus, we 
only use the real part of the Fourier basis. We extend the proof 
technique from previous work to our new setting, and obtain the following result,
\begin{lemma}
For any $2d$-degree polynomial $Q(t) = \sum_{j=0}^{2d} c_j t^j \in \R$, any $T > 0$ and any $\epsilon > 0$, there always exists frequency $f > 0$ (which depends on $d$ and $\epsilon$) and
$
x^*(t) = \sum_{i=1}^{d+1} \alpha_i \cos (2\pi f i t) + \beta_i \cos (2\pi f i t + \theta_i)
$
with coefficients $\{\alpha_i,\beta_i\}_{i=0}^d$ such that
$
\forall t \in [0,T], | x^*(t) - Q(t) | \leq \epsilon.
$
\end{lemma}
\begin{proof}
First, we define $\gamma_j$ as follows
\begin{align*}
\gamma_{2j} =  \sum_{i=1}^{d+1} i^{2j} \alpha_i, \gamma_{2j+1} =  \sum_{i=1}^{d+1} i^{2j+1}\beta_i.
\end{align*}

Using Claim~\ref{cla:fourier_rewrite_x*}, we can rewrite $x^*(t)$
\begin{align*}
x^*(t) = Q(t) + ( P_1(t) - Q(t) ) + P_2(t)
\end{align*}
where {
\begin{align*}
P_1(t) = &\sum_{j=0}^{d} \frac{ (-1)^j }{ (2j) ! } (2 \pi f t)^{2j} \gamma_{2j}  + \sum_{j=0}^d \frac{ (-1)^j }{ (2j+1) ! } (2 \pi f t )^{2j+1} \gamma_{2j+1}, \\
P_2(t) = &\sum_{j=d+1}^{\infty} \frac{ (-1)^j }{ (2j) ! } (2 \pi f t)^{2j} \gamma_{2j} + \sum_{j=d+1}^{\infty} \frac{ (-1)^j }{ (2j+1) ! } (2 \pi f t )^{2j+1} \gamma_{2j+1}.
\end{align*}}
It suffices to show $P_1 (t) = Q(t)$ and $|P_2 (t)| \leq \epsilon, \forall t \in [0,T]$. We first show $P_1 (t) = Q(t)$,
\begin{claim}\label{cla:P_1_is_Q}
For any fixed $f$ and any fixed $2d+2$ coefficients $c_0, c_1, \cdots, c_{2d+1} $, there exists $2d+2$ coefficients $\alpha_0, \alpha_1, \cdots, \alpha_{d}$ and $\beta_0, \beta_1, \cdots, \beta_d$ such that, for all $t$,
$P_1(t) = Q(t)$.
\end{claim}
\begin{proof}
Recall the definition of $Q(t)$ and $P_1(t)$, the problem becomes an regression problem. To guarantee $Q(t) = P_1(t), \forall t$. For any fixed $f$ and coefficients $c_0,\cdots,c_{2d+1}$, we need to solve a linear system with $2d+2$ unknown variables $\gamma_0, \gamma_1, \cdots, \gamma_{2d+1}$ and $2d+2$ constraints : $\forall j \in \{0,1,\cdots, d\}$, 
\begin{align*}
  \frac{(-1)^j}{(2j)!} (2\pi f)^{2j} \gamma_{2j} = c_{2j}, \frac{(-1)^j}{(2j + 1)!} (2\pi f)^{2j+1} \gamma_{2j+1} = c_{2j+1}.
\end{align*}
Further, we have
$
  \frac{(-1)^j}{(2j)!} (2\pi f)^{2j} \sum_{i=1}^{d+1} i^{2j} \alpha_{i-1} = c_{2j}$ and $\frac{(-1)^j}{(2j + 1)!} (2\pi f)^{2j+1} \sum_{i=1}^{d+1} i^{2j+1} \beta_{i-1} = c_{2j+1}.
$

Let $A \in \R^{(d+1) \times (d+1)}$ denote the Vandermonde matrix where $A_{i,j} = (i^2)^j, \forall i,j \in [d+1] \times \{0,1,\cdots,d\}$. Using Lemma~\ref{lem:vandermonde}, we know $\det( A ) \neq 0$, then there must exist a solution to $A \alpha = c_{\text{even}}$.

Let $B \in \R^{(d+1) \times (d+1)}$ denote the Vandermonde matrix where $B_{i,j} = (i^{ (2j+1 / j) } )^j$, $\forall i,j \in [d+1] \times \{0,1,\cdots,d\}$. Using Lemma~\ref{lem:vandermonde}, we know $\det( B ) \neq 0$, then there must exist a solution to $B \beta = c_{\text{odd}}$.
\end{proof}
We can prove $|P_2(t)|\leq \epsilon, \forall t \in [0,T]$.(We defer the proof 
to Claim~\ref{cla:P_2_is_small} in Appendix~\ref{app:expower})
Thus, combining the Claim~\ref{cla:P_2_is_small} with Claim~\ref{cla:P_1_is_Q} completes the proof.
\end{proof}

\subsection{Exponential Functions Have Limited Expressive Power}

Given $k$ coefficients $c_1, \cdots, c_k \in \R$ and $k$ decay parameters $\alpha_1, \cdots, \alpha_k \in (0,1)$, we define function $x(t) = \sum_{i=1}^{k} c_i \alpha_i^t.$
We provide an explicit counterexample which is a degree-$9$ polynomial. Using that example, we are able to show the following result and defer the proof to Appendix~\ref{app:expower}.

\define{thm:exponential_decay_no_expressive_power}{Theorem}{
There is a polynomial $P(t) : \R \rightarrow \R$ with $O(1)$ degree such that, for any $k \geq 1$, for any $x(t) = \sum_{i=1}^{k} c_i \alpha_i^t$, for any $k$ coefficients $c_1,\cdots, c_k \in \R$ and $k$ decay parameters $\alpha_1,\cdots,\alpha_k \in (0,1)$ such that
\begin{align*}
\frac{1}{T} \int_{0}^T | P(t) -x(t)| \mathrm{d} t \gtrsim \frac{1}{T} \int_{0}^T |P(t)| \mathrm{d} t.
\end{align*}
}
\state{thm:exponential_decay_no_expressive_power}

\section{Vanishing and Exploding Gradients}\label{sec:vangrad}
In this section, we analyze the vanishing/exploding gradient issue in various
recurrent architectures. Specifically we give lower and upper bounds of 
gradient magnitude under the linear setting and show that the gradient of FRU does 
not explode or vanish with temporal dimension $T \rightarrow \infty$.
We first analyze RNN and SRU models as a baseline and show their gradients 
vanish/explode exponentially with $T$.
\paragraph{Gradient of linear RNN.}
For linear RNN, we have:
\begin{align*}
  h^{(t+1)} = W \cdot h^{(t)} + U \cdot x^{(t)} + b
\end{align*}
where $t=0,1,2\cdots T-1$. Thus
\begin{align*}
  h^{(T)} = & ~ W \cdot h^{(T-1)} + U \cdot x^{(T-1)} + b \\
  = & ~ W^{T-T_0} \cdot h^{(T_0)} + \sum_{t=T_0}^{T} W^{T-t-1} (U \cdot x^{(t)} + b)
\end{align*} 
Let $L = L(h^{(T)})$ denote the loss function. By Chain rule,
we have{
\begin{align*}
  \left\| \frac{\partial L}{\partial h^{(T_0)} } \right\| = & ~ \left\| \left( 
    \frac{ \partial h^{(T)} }{\partial h^{(T_0)}} \right)^\top \frac{\partial 
  L}{\partial h^{(T)}} \right\|  \\
  = & ~ \left\| (W^{T-T_0})^\top \cdot \frac{\partial L}{\partial h^{(T)} } \right\| \\
  \geq & ~ \sigma_{\min} ( W^{T-T_0}  ) \cdot \left\| \frac{\partial L}{\partial h^{(T)} } \right\|.
\end{align*}}
Similarly for the upper bound:{
\begin{align*}
  \left\| \frac{\partial L}{\partial h^{(T_0)} } \right\| \leq & ~ \sigma_{\max} ( W^{T-T_0}  ) \cdot \left\| \frac{\partial L}{\partial h^{(T)} } \right\|.
\end{align*}}

\paragraph{Gradient of linear SRU.}
For linear SRU, we have:
\begin{align*}
  h^{(t)} = & ~ W_1W_2 \cdot u^{(t-1)} + W_2b_1 + W_3 \cdot x^{(t-1)} +b_2,\\
  u^{(t)} = & ~ \alpha \cdot u^{(t-1)} + (1-\alpha) h^{(t)}.
\end{align*}
Denoting $W = W_1 W_2$ and $B = W_2 b_1 + b_2$, we have
\define{cla:sru_ut_is_sum}{Claim}{
Let $\ov{W}=\alpha I + (1-\alpha) W$. Then using SRU update rule, we have
  $u^{(T)} =  \ov{W}^{T-T_0} u^{(T_0)}  + \sum_{t=T_0}^T \ov{W}^{T-t-1} (1-\alpha) W_3 (  x^{(t)} + B) $.
 }
\state{cla:sru_ut_is_sum}
We provide the proof in Appendix~\ref{app:vangrad}.

With $L=L(u^{(T)})$, by Chain rule, we have the lower bound:{
\begin{align*}
  \left\| \frac{\partial L}{\partial u^{(T_0)} } \right\| = & ~ \left\|  ( (\alpha I + (1-\alpha)W)^\top )^{T-T_0} \frac{\partial L}{\partial u^{(T)}} \right\| \\
  \geq & ~ (\alpha + (1-\alpha)\sigma_{\min}(W))^{T-T_0}  \cdot \left\| \frac{\partial L}{\partial u^{(T)}} \right\|.
\end{align*}}
And similarly for the upper bound:{
\begin{align*}
  \left\| \frac{\partial L}{\partial u^{(T_0)} } \right\| \leq & ~ (\alpha  + (1-\alpha)\sigma_{\max}(W))^{(T-T_0)}  \cdot \left\| \frac{\partial L}{\partial u^{(T)}} \right\|.
\end{align*}}
These bounds for RNN and SRU are achievable, a simple example would be 
$W=\sigma I$. 
It is easy to notice that with $\alpha\in(0,1)$, SRUs have better gradient bounds 
than RNNs. However, SRUs is only better by a constant factor on the exponent and 
gradients for both methods could still explode or vanish 
exponentially with temporal dimension $T$.

\paragraph{Gradient of linear FRU.}
By design, FRU avoids vanishing/exploding gradient by its residual learning 
structure. Specifically, the linear \rm{FRU} has bounded gradient which is independent of the temporal 
dimension $T$. This means no matter how deep the network is, gradient of linear FRU 
would never vanish or explode. We have the following theorem:

\define{cla:fru_ut_is_sum}{Claim}{
$
u^{(T)} = \prod_{t=T_0+1}^T ( I + \frac{1}{T} \cos(2\pi f t/T + \theta)  W) u^{(T_0)} + \frac{1}{T} \sum_{t=T_0+1}^T \cos(2\pi f t / T + \theta) (W_3 x^{(t-1)} + B ).
$
}

\define{cla:fru_upper_bound_gradient}{Claim}{
$
\left\| \frac{\partial L}{ \partial u^{(T_0)} } \right\| \leq e^{ \sigma_{\max} (W) } \left\| \frac{\partial L}{\partial u^{(T)} } \right\|.
$
}

\begin{theorem}\label{thm:fru_gradient}
  With FRU update rule in \eqref{eq:fru_update_rule}, and $\phi$ being identity, we have: {
$
e^{- 2\sigma_{\max}(W_1W_2)} \left\| \frac{\partial L}{\partial u^{(T)}} \right\|
\leq \left\| \frac{\partial L}{ \partial u^{(T_0)} } \right\| \leq e^{ \sigma_{\max} (W_1W_2) } \left\| \frac{\partial L}{\partial u^{(T)}} \right\|
$}
for any $T_0\leq T$.
\end{theorem}

\begin{proof}
For linear FRU, we have:
\begin{align}
  h^{(t)} = & ~ W_1 W_2 \cdot u^{(t-1)} + W_2b_1 + W_3 x^{(t-1)} +b_2, \notag \\
  u^{(t)} = & ~  u^{(t-1)} + \frac{1}{T} \cos( 2\pi ft/T + \theta) h^{(t)}. \notag 
\end{align}
Let $W = W_1 W_2$ and $B = W_2 b_1 + b_2$, we can rewrite $u^{(T)}$ in the following way,
\state{cla:fru_ut_is_sum}
We provide the proof of Claim~\ref{cla:fru_ut_is_sum} in Appendix~\ref{app:vangrad}. By Chain rule{
\begin{align*}
   \left\| \frac{\partial L}{\partial u^{(T_0)}} \right\| 
  = & ~ \left\| \left( \frac{\partial h^{(T)}}{\partial h^{(T_0)}} \right)^\top \frac{\partial L}{\partial h^{(T)}} \right\| \\
  = & ~ \left\| \left(\prod_{t=T_0+1}^T (I + \frac{1}{T} \cos ( 2\pi f t /T + \theta ) \cdot W ) \right)^\top \frac{\partial L}{\partial u^{(T)}} \right\| \\
  \geq & ~ \sigma_{\min} \left( \prod_{t=T_0+1}^T (I + \frac{1}{T} \cos ( 2\pi f t /T + \theta ) \cdot W ) \right) \cdot \left\| \frac{\partial L}{\partial u^{(T)}} \right\| \\
  \geq & ~   \prod_{t=T_0+1}^T \sigma_{\min} (I + \frac{1}{T} \cos ( 2\pi f t /T + \theta ) \cdot W ) \cdot \left\| \frac{\partial L}{\partial u^{(T)}} \right\|.
\end{align*}}
We define two sets $S_-$ and $S_+$ as follows {
\begin{align*}
S_- = \{ t ~|~ \cos (2\pi f t /T + \theta) < 0, t = T_0+1, T_0+2, \cdots, T \}, \\
S_+ = \{ t ~|~ \cos (2\pi f t /T + \theta) \geq 0, t = T_0+1, T_0+2, \cdots, T \}.
\end{align*}}
Thus, we have{
\begin{align*}
  & ~ \prod_{t=T_0+1}^T \sigma_{\min} (I + \frac{1}{T} \cos ( 2\pi f t /T + \theta ) \cdot W ) \\
 = & ~ \prod_{t \in S_+} (1 + \frac{1}{T} \cos ( 2\pi f t /T + \theta ) \cdot \sigma_{\min} (W) ) \\
 & ~ \cdot \prod_{t \in S_-} (1 + \frac{1}{T} \cos ( 2\pi f t /T + \theta ) \cdot \sigma_{\max} (W) )
\end{align*}}
The first term can be easily lower bounded by $1$ and the question is how to lower bound the second term. Since $\sigma_{\max} < T$, we can use the fact $1 - x \geq e^{-2x}, \forall x \in [0,1]$,{
\begin{align*}
& ~\prod_{t \in S_-} (1 + \frac{1}{T} \cos ( 2\pi f t /T + \theta) \cdot \sigma_{\max} (W) ) \\
 \geq & ~ \prod_{t \in S_-} \exp( -2 \frac{1}{T} |\cos(2\pi f t/T + \theta)| \sigma_{\max}(W) ) \\
 = & ~ \exp \left( -2 \sum_{t \in S_-} \frac{1}{T} |\cos(2\pi f t/T + \theta)| \sigma_{\max}(W) \right) \\
 \geq & ~ \exp(- 2\sigma_{\max}(W)),
\end{align*}}
where the last step follows by $|S_-| \leq T$ and 
$|\cos(2\pi f t/T + \theta)| \leq 1$. Therefore, we complete the proof of 
lower bound. Similarly, we can show the following upper bound
\state{cla:fru_upper_bound_gradient}
We provide the proof in Appendix~\ref{app:vangrad}. Combining the lower bound and upper bound together, we complete the proof.
\end{proof}

\section{Experimental Results}\label{sec:exp}

\begin{figure}[!t]
    \centering
    \includegraphics[width=0.9\textwidth]{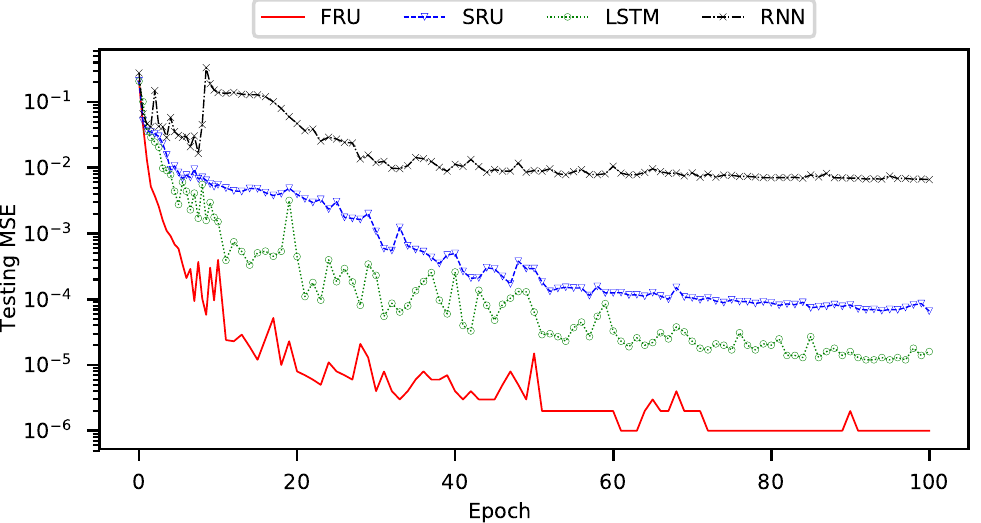}
    \caption{Test MSE of different models on mix-sin synthetic data. FRU uses $\textrm{FRU}_{120, 5}$.}
    \label{fig:sin5_synthetic}
\end{figure}

We implemented the Fourier recurrent unit in \texttt{Tensorflow} \cite{DBLP:tensorflow}
and used the standard implementation of \texttt{BasicRNNCell} and \texttt{BasicLSTMCell} for RNN and LSTM, respectively. 
We also used the released source code of SRU \cite{oliva2017statistical} and used the default configurations of $\{\alpha_i\}_{i=1}^5 = \{0.0, 0.25, 0.5, 0.9, 0.99\}$, $g_t$ dimension of 60, and $h^{(t)}$ dimension of 200. We release our codes on github\footnote{\url{https://github.com/limbo018/FRU}}.
For fair comparison, we construct one layer of above cells with 200 units in the experiments. 
Adam \cite{kingma2014adam} is adopted as the optimization engine. 
We explore learning rates in \{0.001, 0.005, 0.01, 0.05, 0.1\} and learning 
rate decay in \{0.8, 0.85, 0.9, 0.95, 0.99\}. The best results are reported 
after grid search for best hyper parameters.
For simplicity, we use $\textrm{FRU}_{k, d}$ to denote $k$ sampled sparse 
frequencies and $d$ dimensions for each frequency $f_k$ in a FRU cell.

\subsection{Synthetic Data}

We design two synthetic datasets to test our model: mixture of sinusoidal 
functions~(mix-sin) and mixture of polynomials~(mix-poly). 
For mix-sin dataset, we first construct $K$ components with each component 
being a combination of $D$ sinusoidal functions at different frequencies and 
phases~(sampled at beginning). Then, for each data point, we mix the $K$ 
components with randomly sampled weights.
Similarly, each data point in mix-poly dataset is a random mixture of $K$ fixed $D$ 
degree polynomials, with coefficients sampled at beginning and fixed. 
Alg.~\ref{alg:sin5_synthetic} and Alg.~\ref{alg:poly_synthetic} explain these 
procedures in detail. Among the sequences, $80\%$ are used for training and 
$20\%$ are used for testing. We picked sequence length $T$ to be 176, number 
of components $K$ to be 5 and degree $D$ to be 15 for mix-sin and $\{5,10,15\}$ 
for mix-poly.
At each time step $t$, models are asked to predict the sequence value 
at time step $t+1$. It requires the model to learn the $K$ underlying 
functions and uncover the mixture rates at beginning time steps. Thus we can 
measure the model's ability to express sinusoidal and polynomial functions as 
well as their long term memory. 

Figure~\ref{fig:sin5_synthetic} and~\ref{fig:poly_synthetic} plots the testing 
mean square error~(MSE) of different models on mix-sin/mix-poly datasets.
We use learning rate of 0.001 and learing rate decay of 0.9 for training. 
FRU achieves orders of magnitude smaller MSE than other models on mix-sin and mix-poly 
datasets, while using about half the number of 
parameters of SRU. This indicates FRU's ability to easily express these 
component functions.

\begin{figure}[t]
    \centering
    \includegraphics[width=0.9\textwidth]{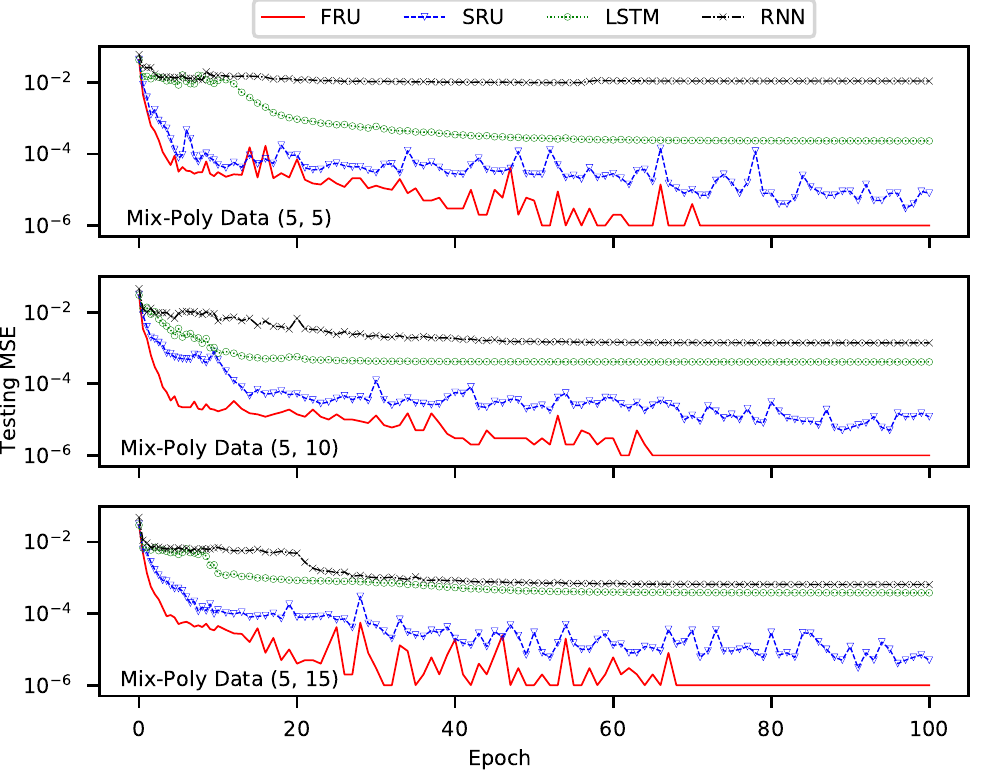}
    \caption{Test MSE of different models on mix-poly synthetic data with different maximum degrees of polynomial basis. FRU uses $\textrm{FRU}_{120, 5}$.}
    \label{fig:poly_synthetic}
\end{figure}

To explicitly demonstrate the gradient stability and ability to learn long 
term dependencies of different models, we analyzed the partial gradient at 
different distance. Specifically, we plot the partial derivative norm of error on digit 
$t$ w.r.t. the initial hidden state, i.e. 
$\frac{\partial (\hat{y}^{(t)}-y^{(t)})^2}{\partial h^{(0)}}$ where $y^{(t)}$ 
is label and $\hat{y}^{(t)}$ is model prediction.
The norms of gradients for FRU are very stable from $t=0$ to $t=300$. 
With the convergence of training, the amplitudes of gradient curves gradually decrease. 
However, the gradients for SRU decrease in orders of magnitudes with the increase of time steps, 
indicating that SRU is not able to capture long term dependencies.
The gradients for RNN/LSTM are even more unstable and the vanishing issues are rather severe. 

\begin{figure*}[!h]
    \centering
    \includegraphics[width=\textwidth]{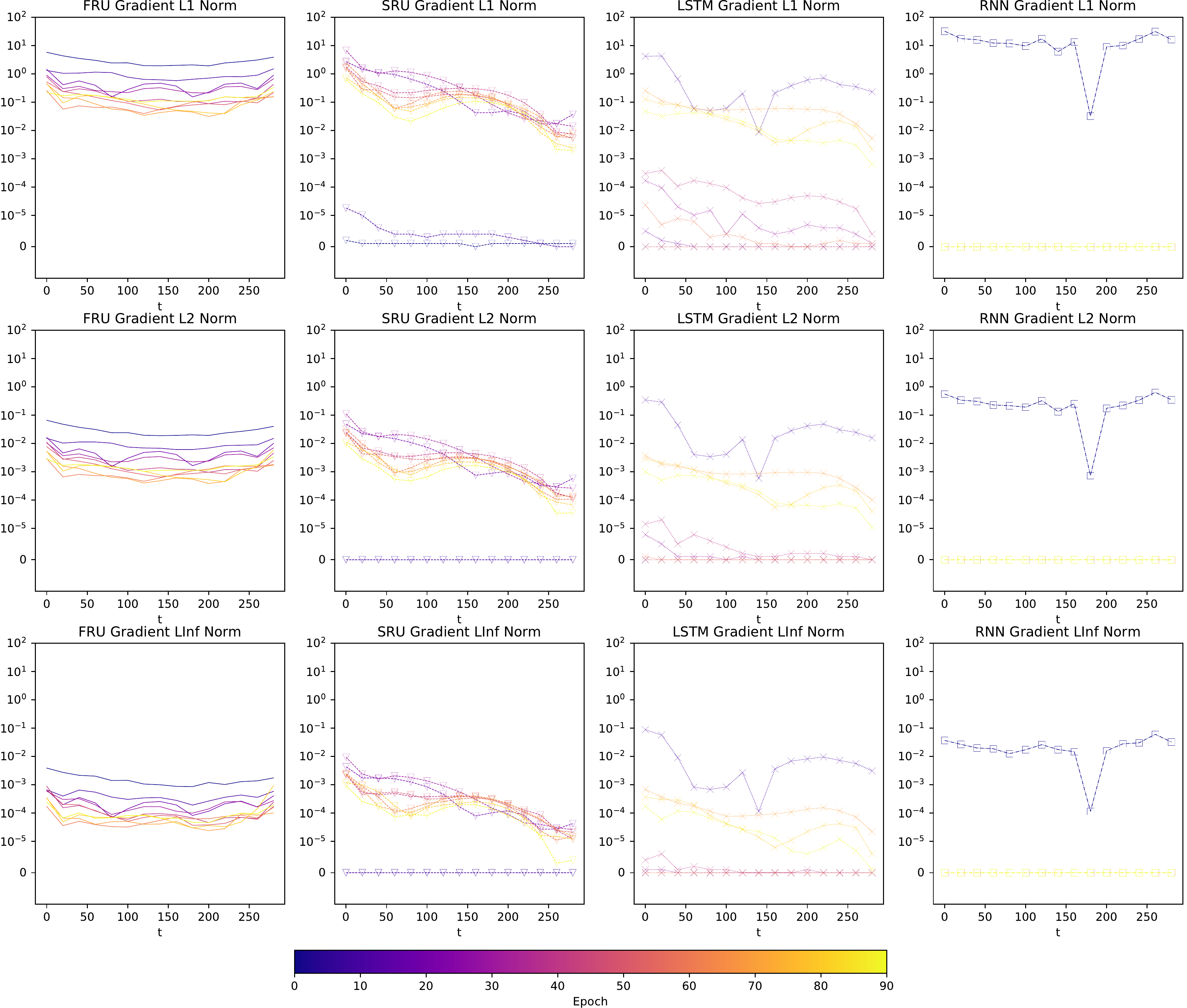}
    \caption{L1, L2, and L$\infty$ norms of gradients for different models on the training of mix-poly (5, 5) dataset.
        We evaluate the gradients of loss to the initial state with time steps, 
        i.e., $\frac{\partial (\hat{y}^{(t)}-y^{(t)})^2}{\partial h^{(0)}}$, 
        where $(\hat{y}^{(t)}-y^{(t)})^2$ is the loss at time step $t$.
        Each point in a curve is averaged over gradients at 20 consecutive time steps. 
        We plot the curves at epoch $0, 10, 20, \dots, 90$ with different colors from dark to light. 
        FRU uses $\textrm{FRU}_{120, 5}$ and SRU uses $\{\alpha_i\}_{i=1}^5 = \{0.0, 0.25, 0.5, 0.9, 0.95\}$.
    }
    \label{fig:poly_synthetic_degree5_npoly5_grads_time}
\end{figure*}



\subsection{Pixel-MNIST Dataset}

We then explore the performance of Fourier recurrent units in classifying MNIST dataset.
Each $28 \times 28$ image is flattened to a long sequence with a length of 784. 
The RNN models are asked to classify the data into 10 categories after being fed 
all pixels sequentially.
Batch size is set to 256 and dropout \cite{srivastava2014dropout} is not included in this experiment. 
A softmax function is applied to the 10 dimensional output at last layer of 
each model.
For FRU, frequencies $f$ are uniformly sampled in log space from 0 to 784.

\begin{figure}[t]
    \centering 
    \includegraphics[width=0.9\textwidth]{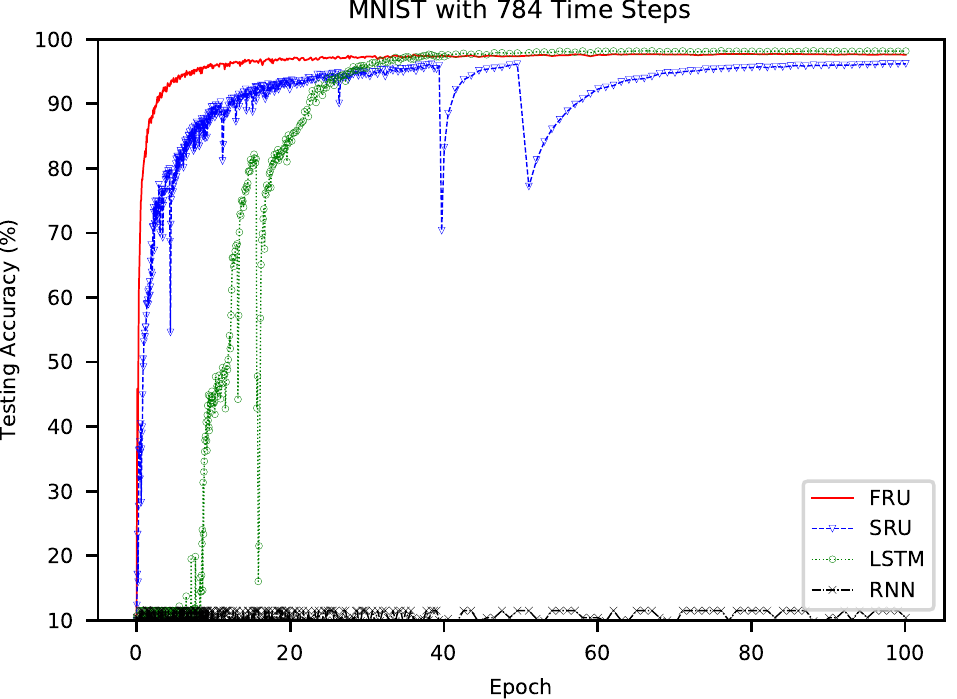}
    \caption{Testing accuracy of RNN, LSTM, SRU, and FRU for pixel-by-pixel MNIST dataset. 
    FRU uses $\textrm{FRU}_{60,10}$, i.e., 60 frequencies with the dimension of each frequency $f_k$ to be 10.}
    \label{fig:MNIST784}
\end{figure}

\DefMacro{MNIST784.RNN.var}{42K}
\DefMacro{MNIST784.RNN.var.ratio}{0.26}
\DefMacro{MNIST784.LSTM.var}{164K}
\DefMacro{MNIST784.LSTM.var.ratio}{1.00}
\DefMacro{MNIST784.SRU.var}{275K}
\DefMacro{MNIST784.SRU.var.ratio}{1.68}
\DefMacro{MNIST784.FRU.40.10.var}{107K}
\DefMacro{MNIST784.FRU.40.10.var.ratio}{0.65}
\DefMacro{MNIST784.FRU.60.10.var}{159K}
\DefMacro{MNIST784.FRU.60.10.var.ratio}{0.97}

\DefMacro{MNIST784.RNN.test}{10.39}
\DefMacro{MNIST784.LSTM.test}{98.17}
\DefMacro{MNIST784.SRU.test}{96.20}
\DefMacro{MNIST784.FRU.40.10.test}{96.88}
\DefMacro{MNIST784.FRU.60.10.test}{97.61}

\begin{table}[!t]
\centering
\caption{Testing Accuracy of MNIST Dataset}
\label{tab:MNIST784}
\small
\begin{tabular}{|c|c|c|c|}
\hline
Networks               & \begin{tabular}[c]{@{}c@{}}Testing\\ Accuracy\end{tabular} & \#Variables & \begin{tabular}[c]{@{}c@{}}Variable \\ Ratio\end{tabular} \\ \hline
RNN                    & \UseMacro{MNIST784.RNN.test}\%                                                    & \UseMacro{MNIST784.RNN.var}       & \UseMacro{MNIST784.RNN.var.ratio}       \\ \hline
LSTM                   & \UseMacro{MNIST784.LSTM.test}\%                                                   & \UseMacro{MNIST784.LSTM.var}      & \UseMacro{MNIST784.LSTM.var.ratio}      \\ \hline
SRU                    & \UseMacro{MNIST784.SRU.test}\%                                                    & \UseMacro{MNIST784.SRU.var}       & \UseMacro{MNIST784.SRU.var.ratio}       \\ \hline
$\textrm{FRU}_{40,10}$ & \UseMacro{MNIST784.FRU.40.10.test}\%                                              & \UseMacro{MNIST784.FRU.40.10.var} & \UseMacro{MNIST784.FRU.40.10.var.ratio} \\ \hline
$\textrm{FRU}_{60,10}$ & \UseMacro{MNIST784.FRU.60.10.test}\%                                              & \UseMacro{MNIST784.FRU.60.10.var} & \UseMacro{MNIST784.FRU.60.10.var.ratio} \\ \hline
\end{tabular}
\end{table}

Fig.~\ref{fig:MNIST784} plots the testing accuracy of different models during training. 
RNN fails to converge and LSTM converges very slow. 
The fastest convergence comes from FRU, which achieves over 97.5\% accuracy in 
10 epochs while LSTM reaches 97\% at around 40th epoch.
Table~\ref{tab:MNIST784} shows the accuracy at the end of 100 epochs for RNN, LSTM, SRU, and different configurations of FRU. 
LSTM ends up with \UseMacro{MNIST784.LSTM.test}\% in testing accuracy and SRU achieves \UseMacro{MNIST784.SRU.test}\%. 
Different configurations of FRU with 40 and 60 frequencies provide close accuracy to LSTM. 
The number and ratio  of trainable parameters are also illustrated in the table. 
The amount of variables for FRU is much smaller than that of SRU, and comparable to that of LSTM, 
while it is able to achieve smoother training and high testing accuracy. 
We ascribe such benefits of FRU to better expressive power and more robust to gradient vanishing from the Fourier representations.



\subsection{Permuted MNIST Dataset}

\begin{figure}[t]
    \centering 
    \includegraphics[width=0.9\textwidth]{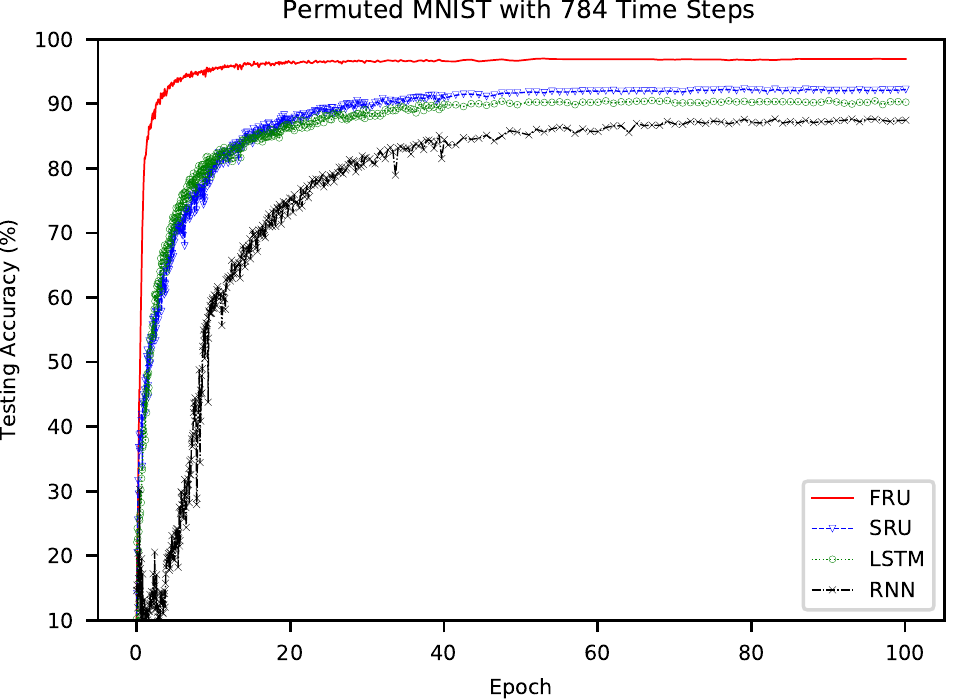}
    \caption{Testing accuracy of RNN, LSTM, SRU, and FRU for permuted pixel-by-pixel MNIST. 
    FRU uses 60 frequencies with the dimension of 10 for each frequency.}
    \label{fig:PermuteMNIST784}
\end{figure}

\DefMacro{PermuteMNIST784.RNN.test}{87.46}
\DefMacro{PermuteMNIST784.LSTM.test}{90.26}
\DefMacro{PermuteMNIST784.SRU.test}{92.21}
\DefMacro{PermuteMNIST784.FRU.test}{96.93}

\begin{table}[t]
    \centering
    \caption{Testing Accuracy of Permuted MNIST Dataset}
    \label{tab:PermuteMNIST784}
    \begin{tabular}{|c|c|c|c|}
        \hline
        RNN     & LSTM    & SRU     & FRU     \\ \hline
        \UseMacro{PermuteMNIST784.RNN.test}\% & \UseMacro{PermuteMNIST784.LSTM.test}\% & \UseMacro{PermuteMNIST784.SRU.test}\% & \UseMacro{PermuteMNIST784.FRU.test}\% \\ \hline
    \end{tabular}
\end{table}

We now use the same models as previous section and test on permuted MNIST dataset. 
Permute MNIST dataset is generated from pixel-MNIST dataset with a random but 
fixed permutation among its pixels.
It is reported the permutation increases the difficulty of classification \cite{arjovsky2016unitary}. 
The training curve is plotted in Fig.~\ref{fig:PermuteMNIST784} and the converged accuracy is shown in Table~\ref{tab:PermuteMNIST784}. 
We can see that in this task, FRU can achieve \DeltaScale{\UseMacro{PermuteMNIST784.FRU.test}}{\UseMacro{PermuteMNIST784.SRU.test}}{1}\% higher accuracy than SRU,  
\DeltaScale{\UseMacro{PermuteMNIST784.FRU.test}}{\UseMacro{PermuteMNIST784.LSTM.test}}{1}\% higher accuracy than LSTM, 
and \DeltaScale{\UseMacro{PermuteMNIST784.FRU.test}}{\UseMacro{PermuteMNIST784.RNN.test}}{1}\% higher accuracy than RNN. 
The training curve of FRU is smoother and converges much faster than other models. 
The benefits of FRU to SRU are more significant in permuted MNIST than that in 
the original pixel-by-pixel MNIST. This can be explained by higher model 
complexity of permuted-MNIST and stronger expressive power of FRU. 



\subsection{IMDB Dataset}

\begin{figure}[t]
    \centering 
    \includegraphics[width=0.9\textwidth]{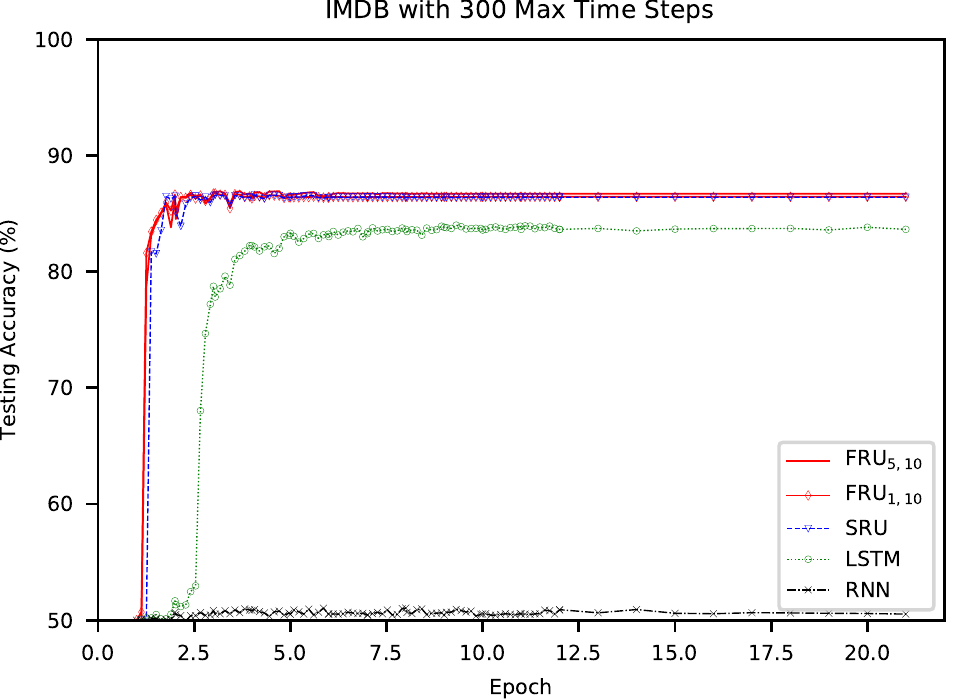}
    \caption{Testing accuracy of RNN, LSTM, SRU, and FRU for IMDB dataset. 
    $\textrm{FRU}_{5, 10}$ uses 5 frequencies with the dimension of 10 for each frequency $f_k$.
    $\textrm{FRU}_{1, 10}$ is an extreme case of FRU with only frequency 0. 
}
    \label{fig:IMDB300}
\end{figure}

\DefMacro{IMDB300.RNN.var}{33K}
\DefMacro{IMDB300.RNN.var.ratio}{0.25}
\DefMacro{IMDB300.LSTM.var}{132K}
\DefMacro{IMDB300.LSTM.var.ratio}{1.00}
\DefMacro{IMDB300.SRU.var}{226K}
\DefMacro{IMDB300.SRU.var.ratio}{1.72}
\DefMacro{IMDB300.FRU.var}{12K}
\DefMacro{IMDB300.FRU.var.ratio}{0.09}
\DefMacro{IMDB300.FRU.1.10.var}{4K}
\DefMacro{IMDB300.FRU.1.10.var.ratio}{0.03}

\DefMacro{IMDB300.RNN.test}{50.53}
\DefMacro{IMDB300.LSTM.test}{83.64}
\DefMacro{IMDB300.SRU.test}{86.40}
\DefMacro{IMDB300.FRU.test}{86.71}
\DefMacro{IMDB300.FRU.1.10.test}{86.44}

\begin{table}[t]
    \centering
    \caption{Testing Accuracy of IMDB Dataset}
    \label{tab:IMDB300}
    \begin{tabular}{|c|c|c|c|}
        \hline
    Networks & \begin{tabular}[c]{@{}c@{}}Testing\\ Accuracy\end{tabular} & \#Variables & \begin{tabular}[c]{@{}c@{}}Variable \\ Ratio\end{tabular} \\ \hline
        RNN                 & \UseMacro{IMDB300.RNN.test}\%                          & \UseMacro{IMDB300.RNN.var}       & \UseMacro{IMDB300.RNN.var.ratio}       \\ \hline
        LSTM                & \UseMacro{IMDB300.LSTM.test}\%                         & \UseMacro{IMDB300.LSTM.var}      & \UseMacro{IMDB300.LSTM.var.ratio}      \\ \hline
        SRU                 & \UseMacro{IMDB300.SRU.test}\%                          & \UseMacro{IMDB300.SRU.var}       & \UseMacro{IMDB300.SRU.var.ratio}       \\ \hline
$\textrm{FRU}_{5, 10}$      & \UseMacro{IMDB300.FRU.test}\%                          & \UseMacro{IMDB300.FRU.var}       & \UseMacro{IMDB300.FRU.var.ratio}       \\ \hline
$\textrm{FRU}_{1, 10}$      & \UseMacro{IMDB300.FRU.1.10.test}\%                     & \UseMacro{IMDB300.FRU.1.10.var}  & \UseMacro{IMDB300.FRU.1.10.var.ratio}  \\ \hline
\end{tabular}
\end{table}

We further evaluate FRU and other models with IMDB movie review dataset (25K training and 25K testing sequences). 
We integrate FRU and SRU into \texttt{TFLearn} \cite{tflearn2016}, a high-level API for \texttt{Tensorflow}, and test together with LSTM and RNN. 
The average sequence length of the dataset is around 284 and the maximum 
sequence length goes up to over 2800. We truncate all sequences to a length 
of 300. 
All models use a single layer with 128 units, batch size of 32, dropout keep rate of 80\%. 
FRU uses 5 frequencies with the dimension for each frequency $f_k$ as 10. 
Learning rates and decays are tuned separately for each model for best performance.

Fig.~\ref{fig:IMDB300} plots the testing accuracy of different models during training 
and Table~\ref{tab:IMDB300} gives the eventual testing accuracy. 
$\textrm{FRU}_{5, 10}$ can achieve \DeltaScale{\UseMacro{IMDB300.FRU.test}}{\UseMacro{IMDB300.SRU.test}}{1}\% higher accuracy than SRU, 
and \DeltaScale{\UseMacro{IMDB300.FRU.test}}{\UseMacro{IMDB300.LSTM.test}}{1}\% better accuracy than LSTM. 
RNN fails to converge even after a large amount of training steps. 
We draw attention to the fact that with 5 frequencies, FRU achieves the highest accuracy with 10X fewer variables than LSTM and 19X fewer variables than SRU, 
indicating its exceptional expressive power. 
We further explore a special case of FRU, $\textrm{FRU}_{1, 10}$, with only frequency 0, which is reduced to a RNN-like cell. 
It uses 8X fewer variables than RNN, but converges much faster and is able to achieve the second highest accuracy. 


Besides the experimental results above, Section~\ref{sec:More_Experimental_Results} in Appendix provides more experiments on 
different configurations of FRU for MNIST dataset, 
detailed procedures to generate synthetic data, 
and study of gradient vanishing during training. 

\section{Conclusion}\label{sec:conclu}
In this paper, we have proposed a simple recurrent architecture called the 
Fourier recurrent unit~(FRU), which 
has the structure of residual learning and exploits the expressive power 
of Fourier basis. We gave a proof of the expressivity of sparse Fourier basis 
and showed that FRU does not suffer from vanishing/exploding gradient in the linear 
case. Ideally, due to the global support of Fourier basis, FRU is able to 
capture dependencies of any length. We empirically showed FRU's ability to 
fit mixed sinusoidal and polynomial curves, and FRU 
outperforms LSTM and SRU on pixel MNIST dataset with fewer parameters. On 
language models datasets, FRU also shows its superiority over other RNN 
architectures. Although we now limit our models to recurrent structure, it would be
very exciting to extend the Fourier idea to help gradient issues/expressive power
for non-recurrent deep neural network, e.g. MLP/CNN.
It would also be interesting to 
see how other basis functions, such as polynomial basis, will behave on 
similar architectures. For example,
Chebyshev's polynomial is one of the interesting case to try.

\clearpage
\newpage

\bibliographystyle{alpha}
\bibliography{ref}

\newpage
\appendix
\newpage
\section*{Appendix}
\section{Preliminaries}\label{app:preli}
In this section we prove the equations \eqref{eq:sru_summary} and~\eqref{eq:fru_summary} and include more 
background of Sparse Fourier transform.
\begin{claim}\label{cla:sru_vTi_is_sum}
  With the SRU update rule in \eqref{eq:sru_update_rule}, for $i\in [k]$ we have:
  \begin{align*}
    v^{(t)}_{i} = \alpha_i^{t} \cdot v^{(0)}_{i} + (1-\alpha_i) \sum_{\tau=1}^{t} \alpha_i^{t-\tau} \cdot h^{(\tau)}
  \end{align*}
\end{claim}

\begin{proof}
We have 
\begin{align*}
  v_{i}^{(t)} 
= & ~ \alpha_i \cdot v_{i}^{(t-1)} + (1-\alpha_i) \cdot h^{(t)} \\
= & ~ \alpha_i( \alpha_i \cdot v_{i}^{(t-2)} + (1-\alpha_i) h^{(t-1)} ) + (1-\alpha_i) h^{(t)} \\
= & ~ \alpha_i^2 \cdot v_{i}^{(t-2)} + \alpha_i (1-\alpha_i) h^{(t-1)} + (1-\alpha_i) h^{(t)} \\
= & ~ \alpha_i^3 \cdot v_{i}^{(t-3)} + (1-\alpha_i) ( \alpha_i^2 h^{(t-2)} + \alpha_i h^{(t-1)} + h^{(t)} ) \\
= & ~ \cdots \\
= & ~ \alpha_i^{t} \cdot v_{i}^{(0)} + (1-\alpha_i) \sum_{\tau=1}^t \alpha_i^{t-\tau} \cdot h^{(\tau)},
\end{align*}
where the first step follows by definition of $v_{i}^{(t)}$, the second step follows by definition of $v_{i}^{(t-1)}$, and last step follows by applying the update rule (Eq.~\eqref{eq:sru_update_rule}) recursively.
\end{proof}


\begin{claim}\label{cla:fru_vTi_is_sum}
  With the FRU update rule in \eqref{eq:fru_update_rule}, we have:
  \begin{align*}
    u^{(t)}_i = u^{(0)}_i + \frac{1}{T} \sum_{\tau=1}^t \cos(\frac{2\pi f \tau}{T} + \theta) \cdot h^{(\tau)}
  \end{align*}
\end{claim}
\begin{proof}
\begin{align*}
 v_{i}^{(t)} = & ~ v_{i}^{(t-1)} + \frac{1}{T} \cos (2\pi f_i t /t + \theta_i) \cdot h^{(t)} \\
= & ~ v_{i}^{(t-2)} + \frac{1}{T} \cos (2\pi f_i (t-1) /t + \theta_i ) \cdot h^{(t-1)} + \frac{1}{T} \cos (2\pi f_i t / t + \theta_i ) \cdot h^{(t)} \\
= & ~ v_{i}^{(t-3)} + \frac{1}{T} \cos (2\pi f_i (t-2) /t + \theta_i ) \cdot h^{(t-2)} + \frac{1}{T} \cos (2\pi f_i (t-1) /t + \theta_i ) \cdot h^{(t-1)} \\
& ~ + \frac{1}{T} \cos (2\pi f_i t / t + \theta_i ) \cdot h^{(t)} \\
= & ~ \cdots \\
= & ~ v_{i}^{(0)} + \sum_{\tau = 1}^t \frac{1}{T} \cos(2\pi f_i \tau / t + \theta_i ) \cdot h^{(\tau)},
\end{align*}
where the first step follows by definition of $v_{i}^{(t)}$, the second step follows by definition of $v_{i}^{(t-2)}$, and the last step follows by applying the update rule~\eqref{eq:fru_update_rule} recursively.
\end{proof}

\subsection{Background of Sparse Fourier Transform}

The Fast Fourier transform (FFT) is widely used for many real applications.
FFT only requires $O(n \log n)$ time, when $n$ is the data size. Over the last decade,
there is a long line of work (e.g., \cite{hikp12a,hikp12b,ikp14,ik14}) aiming to improve the running time to nearly linear in $k$ where $k \ll n$ is the sparsity. For more details, we refer the readers to \cite{p13}. However, all of the previous work are targeting on the discrete Fourier transform. It is unclear how to solve the $k$-sparse Fourier transform in the continuous setting or even how to define the problem in the continuous setting. Price and Song \cite{ps15} proposed a reasonable way to define the problem, suppose $x^*(t) = \sum_{i=1}^k v_i e^{2\pi \i f_i t}$ is $k$-Fourier-sparse signal, we can observe some signal $x^*(t)+ g(t)$ where $g(t)$ is noise, the goal is to first find $(v'_i,f'_i)$ and then output a new signal $x'(t)$ such that 
\begin{align}\label{eq:fourier_guarantee}
\frac{1}{T} \int_0^T | x'(t) - x(t) |^2 \mathrm{d} t \leq O(1)\frac{1}{T} \int_0^T |g(t)|^2 \mathrm{d}t.
\end{align}
 \cite{ps15} provided an algorithm that requires $ T \geq (\poly\log k) / \eta$, and takes samples/running time in nearly linear in $k$ and logarithmically in other parameters, e.g. duration $T$, band-limit $F$, frequency gap $\eta$. \cite{m15} proved that if we want to find $(v'_i,f'_i)$ which is very close to $(v_i,f_i)$, then $T = \Omega(1/\eta)$. Fortunately, it is actually possible that interpolating the signal in a nice way (satisfying Eq.~\eqref{eq:fourier_guarantee}) without approximating the ground-truth parameters $(v_i,f_i)$ and requiring $T$ is at least the inverse of the frequency gap \cite{ckps16}.

\section{Expressive Power}\label{app:expower}

\begin{claim}\label{cla:P_2_is_small}
$| P_2(t) | \leq \epsilon, \forall t \in [0,T]$.
\end{claim}
\begin{proof}
Using Lemma~\ref{lem:vandermonde}, we can upper bound the determinant of matrix $A$ and $B$ in the following sense,
\begin{align*}
\det(A) \leq 2^{O(d^2 \log d)} \text{~and~} \det(B) \leq 2^{O(d^2 \log d)}.
\end{align*}
 Thus $\sigma_{\max}(A) \leq 2^{O(d^2\log d)}$ and then
\begin{align*}
\sigma_{\min}(A) = \frac{\det(A)}{\prod_{i=1}^{d-1} \sigma_i} \geq 2^{-O(d^3\log d)}.
\end{align*}
Similarly, we have $\sigma_{\max}\leq 2^{O(d^2\log d)}$ and $\sigma_{\min} \geq 2^{-O(d^3\log d)}$. Next we show how to upper bound $|\alpha_i|$
\begin{align*}
\max_{i\in [d+1]} |\alpha_i | \leq & ~ \| \alpha \|_2 \\ 
= & ~ \| A^\dagger c_{\text{even}} \|_2 \\
\leq & ~ \| A^\dagger \|_2 \cdot \| c_{\text{even}} \|_2 \\
\leq & ~ \frac{1}{\sigma_{\min}(A)} \sqrt{d+1} \max_{0\leq j \leq d} \frac{ |c_{2j}| (2j)! }{ (2\pi f)^{2j}}
\end{align*}
Similarly, we can bound $|\beta_i|$,
\begin{align*}
\max_{i\in [d+1]} |\beta_i | \leq & ~ \| \beta \|_2 \\ 
= & ~ \| B^\dagger c_{\text{odd}} \|_2 \\
\leq & ~ \| B^\dagger \|_2 \cdot \| c_{\text{odd}} \|_2 \\
\leq & ~ \frac{1}{\sigma_{\min}(B)} \sqrt{d+1} \max_{0\leq j \leq d} \frac{ |c_{2j+1} | (2j+1)! }{ (2\pi f)^{2j+1}}
\end{align*}
Let $P_{\alpha}(t)$ and $P_{\beta}(t)$ be defined as follows
\begin{align*}
P_{\alpha}(t) = & ~ \sum_{j=d+1}^{\infty} \frac{ (-1)^j }{ (2j)! } (2\pi f t)^{2j} \sum_{i=1}^{d+1} \alpha_i \cdot i^{2j} \\
P_{\beta}(t)  = & ~ \sum_{j=d+1}^{\infty} \frac{ (-1)^j }{ (2j+1)! } (2\pi f t)^{2j+1} \sum_{i=1}^{d+1} \alpha_i \cdot i^{2j+1}
\end{align*}
By triangle inequality, we have
\begin{align*}
|P_2(t)| \leq | P_{\alpha}(t) | + | P_{\beta}(t) |
\end{align*}
We will bound the above two terms in the following way
\begin{align*}
| P_{\alpha}(t) | \leq & ~ \sum_{j=d+1}^{\infty} \frac{ (2\pi f t)^{2j} }{ (2j) ! } \sum_{i=1}^{d+1} | \alpha_i |\cdot i^{2j}  \\
\leq & ~ \sum_{j=d+1}^{\infty} \frac{ (2\pi f t)^{2j} }{ (2j) ! } (d+1)^{d+1} \max_{i \in [d+1]} |\alpha_i| \\
\leq & ~ \sum_{j=d+1}^{\infty} \frac{ (2\pi f t)^{2j} }{ (2j) ! } (d+1)^{d+2}\\
& ~ \cdot \frac{1}{\sigma_{\min}(A)} \frac{(2d)!}{ (2\pi f)^{2d}} \max_{ 0\leq j \leq d } |c_{2j}|\\
\leq & ~ \epsilon,
\end{align*}
where the last step follows by choosing sufficiently small $f$, $f \lesssim \epsilon / ( 2^{\Theta(d^3 \log d)} \max_{0\leq j \leq d} |c_{2j}| )$. Similarly, if $f \lesssim \epsilon / ( 2^{\Theta(d^3 \log d)} \max_{0\leq j \leq d} |c_{2j+1}| )$, we have $|P_{\beta}(t)| \leq \epsilon$.

Putting it all together, we complete the proof of this Claim.
\end{proof}


\begin{claim}\label{cla:fourier_rewrite_x*}
Let $P_1(t)$ and $P_2(t)$ be defined as follows
\begin{align*}
P_1(t) = & ~ \sum_{j=0}^{d} \frac{ (-1)^j }{ (2j) ! } (2 \pi f t)^{2j} \gamma_{2j} + \sum_{j=0}^d \frac{ (-1)^j }{ (2j+1) ! } (2 \pi f t )^{2j+1} \gamma_{2j+1} \\
P_2(t) = & ~ \sum_{j=d+1}^{\infty} \frac{ (-1)^j }{ (2j) ! } (2 \pi f t)^{2j} \gamma_{2j} + \sum_{j=d+1}^{\infty} \frac{ (-1)^j }{ (2j+1) ! } (2 \pi f t )^{2j+1} \gamma_{2j+1}
\end{align*}
Then
\begin{align*}
x^*(t) = Q(t) + ( P_1(t) - Q(t) ) + P_2(t).
\end{align*}
\end{claim}
\begin{proof}
\begin{align*}
 & ~ x^*(t) \\
= & ~ \sum_{i=1}^{d+1} \alpha_i \cos ( 2 \pi f i t ) + \beta_i \sin ( 2\pi f i t ) \\
= & ~ \sum_{i=1}^{d+1} \alpha_i \sum_{j=0}^{\infty} \frac{(-1)^j}{ (2j) ! } (2\pi f i t)^{2j} 
 + \sum_{i=1}^{d+1} \beta_i \sum_{j=0}^{\infty} \frac{ (-1)^j }{ (2j+1) ! } (2\pi f i t)^{2j+1} \\
= & ~ \sum_{j=0}^{\infty} \frac{ (-1)^j }{ (2j) ! } (2\pi f t)^{2j}  \sum_{i=1}^{d+1} i^{2j} \alpha_i 
 + \sum_{j=0}^{\infty} \frac{ (-1)^j }{ (2j+1) ! } (2\pi f t)^{2j+1} \sum_{i=1}^{d+1} i^{2j+1}\beta_i \\
= & ~ \sum_{j=0}^{\infty} \frac{ (-1)^j }{ (2j) ! } (2\pi f t)^{2j}  \gamma_{2j} 
+ \sum_{j=0}^{\infty} \frac{ (-1)^j }{ (2j+1) ! } (2\pi f t)^{2j+1} \gamma_{2j+1} \\ 
= & ~ Q(t) + (P_1(t) - Q(t)) + P_2(t).
\end{align*}
where the first step follows by definition of $x^*(t)$, the second step follows by Taylor expansion, the third step follows by swapping two sums, the fourth step follows by definition of $\gamma_j$, the last step follows by definition of $P_1(t)$ and $P_2(t)$.
\end{proof}


\restate{thm:exponential_decay_no_expressive_power}

\begin{proof}
    Let $f(t) = t - \frac{t^3}{3!} + \frac{t^5}{5!} - \frac{t^7}{7!} + \frac{t^9}{9!}$. 
    We choose $P(t) = f(\beta \frac{t}{T} - \frac{\beta}{2})$, as shown in Fig.~\ref{fig:P(t)}, 
    where $f(-\frac{\beta}{2}) = 0$ and $f(\frac{\beta}{2}) = 0$. 
    $P(t) = 0$ has 5 real solutions $t = 0, \beta_1 T, \frac{1}{2} T, \beta_2 T, T$ between $[0, T]$. 
    Numerically, $\beta = 9.9263 \pm 10^{-4}, \beta_1 = 0.1828 \pm 10^{-4}, \beta_2 = 0.8172 \pm 10^{-4}$. 
    We have 4 partitions for $P(t)$,  
    \begin{itemize}
        \item[$A_1:$] $ P(t) > 0, \forall t \in (0, \beta_1 T)$ and $M_1 = \int_{0}^{\beta_1 T} P(t) \mathrm{d} t = (0.0847 \pm 10^{-4}) T$;  
        \item[$A_2:$] $ P(t) < 0, \forall t \in (t_1, \frac{1}{2} T)$ and $M_2 = \int_{\beta_1 T}^{\frac{1}{2} T} P(t) \mathrm{d} t = (-0.2017 \pm 10^{-4}) T$;  
        \item[$A_3:$] $ P(t) > 0, \forall t \in (\frac{1}{2} T, t_2)$ and $M_3 = \int_{\frac{1}{2} T}^{\beta_2 T} P(t) \mathrm{d} t = (0.2017 \pm 10^{-4}) T$;  
        \item[$A_4:$] $ P(t) < 0, \forall t \in (t_2, T)$ and $M_4 = \int_{\beta_2 T}^{T} P(t) \mathrm{d} t = (-0.0847 \pm 10^{-4}) T$.  
    \end{itemize}
    The integral of $|P(t)|$ across $(0, T)$ is, 
    \begin{align*}
        \int_{0}^{T} |P(t)| \mathrm{d} t = |M_1| + |M_2| + |M_3| + |M_4|  = (0.5727 \pm 10^{-4}) T. 
    \end{align*}
    According to the properties of sums of exponential functions in \cite{shestopaloff2010sums, shestopaloff2008properties}, $x(t) = 0$ has at most two solutions. 
    We consider the integral of fitting error, 
    \begin{align*}
        A = & ~ \int_0^T \left|x(t) - P(t)\right| \mathrm{d} t \\
        = & ~ \underbrace{\int_0^{\beta_1 T} |x(t) - P(t)|\mathrm{d}t}_{A_1}
        + \underbrace{\int_{\beta_1 T}^{\frac{T}{2}} |x(t) - P(t)|\mathrm{d}t}_{A_2} \\
        & ~ + \underbrace{\int_{\frac{T}{2}}^{\beta_2 T} |x(t) - P(t)|\mathrm{d}t}_{A_3}
        + \underbrace{\int_{\beta_2 T}^T |x(t) - P(t)|\mathrm{d}t}_{A_4}. 
    \end{align*}

    {\bf Case 1. } $x(t) = 0$ has zero solution and $x(t) > 0$. 
    \begin{align*}
        \int_0^T |x(t) - P(t)| \mathrm{d} t \ge & ~ |M_2| + |M_4| \\
                                            \ge & ~ (0.2864 - 2 \times 10^{-4}) T. 
    \end{align*}

    {\bf Case 2. } $x(t) = 0$ has zero solution and $x(t) < 0$. 
    \begin{align*}
        \int_0^T |x(t) - P(t)| \mathrm{d} t \ge & ~ |M_1| + |M_3| \\
                                            \ge & ~ (0.2864 - 2 \times 10^{-4}) T. 
    \end{align*}

    {\bf Case 3. } $x(t) = 0$ has one solution $t_1$, $x(t) > 0, \forall t \in (0, t_1)$, and $x(t) < 0, \forall t \in (t_1, T)$. 
    Even if $x(t)$ can fit perfectly in $(0, t_1)$ with $A_1$ and in $(t_1, T)$ with $A_2, A_4$, it cannot fit $A_3$ well. 
    Similarly, even if $x(t)$ can fit perfectly in $(0, t_1)$ with $A_1, A_3$ and in $(t_1, T)$ with $A_4$, it cannot fit $A_2$ well. 
    \begin{align*}
        \int_0^T |x(t) - P(t)| \mathrm{d} t \ge & ~ \min(|M_2|, |M_3|) \\
                                            \ge & ~ (0.2017 - 10^{-4}) T. 
    \end{align*}

    {\bf Case 4. } $x(t) = 0$ has one solution $t_1$, $x(t) < 0, \forall t \in (0, t_1)$, and $x(t) > 0, \forall t \in (t_1, T)$. 
    Even if $x(t)$ can fit perfectly in $(0, t_1)$ with $A_2$ and in $(t_1, T)$ with $A_3$, it cannot fit $A_1, A_4$ well. 
    \begin{align*}
        \int_0^T |x(t) - P(t)| \mathrm{d} t \ge & ~ |M_1| + |M_4| \\
                                            \ge & ~ (0.1694 - 2 \times 10^{-4}) T. 
    \end{align*}

    {\bf Case 5. } $x(t) = 0$ has two solutions $t_1, t_2$, $x(t) > 0, \forall t \in (0, t_1)$, $x(t) < 0, \forall t \in (t_1, t_2)$, and $x(t) > 0, \forall t \in (t_2, T)$. 
    Even if $x(t)$ can fit perfectly with $A_1, A_2, A_3$, it cannot fit $A_4$ well. 
    \begin{align*}
        \int_0^T |x(t) - P(t)| \mathrm{d} t \ge & ~ |M_4| \\
                                            \ge & ~ (0.0847 - 10^{-4}) T. 
    \end{align*}

    {\bf Case 6. } $x(t) = 0$ has two solutions $t_1, t_2$, $x(t) < 0, \forall t \in (0, t_1)$, $x(t) > 0, \forall t \in (t_1, t_2)$, and $x(t) < 0, \forall t \in (t_2, T)$. 
    Even if $x(t)$ can fit perfectly with $A_2, A_3, A_4$, it cannot fit $A_1$ well. 
    \begin{align*}
        \int_0^T |x(t) - P(t)| \mathrm{d} t \ge & ~ |M_1| \\
                                            \ge & ~ (0.0847 - 10^{-4}) T. 
    \end{align*}

    Therefore, for all cases, $x(t)$ cannot achieve constant approximation to the signal $P(t)$, e.g., 
    \begin{align*}
        \int_0^T |x(t) - P(t)| \mathrm{d} t \gtrsim \int_0^T |P(t) |\mathrm{d} t.
    \end{align*}
\end{proof}

\begin{figure*}[!t]
    \centering
    \includegraphics[width=0.8\textwidth]{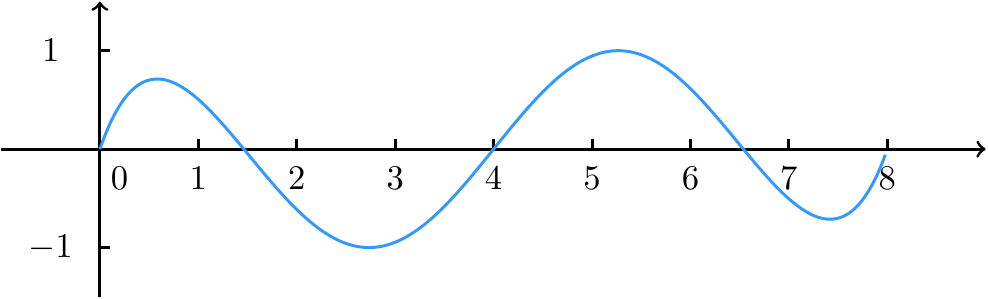}
    \caption{Polynomial function $P(t)$ with $T = 8$.}
    \label{fig:P(t)}
\end{figure*}

\section{Bounded Gradients}\label{app:vangrad}

\restate{cla:sru_ut_is_sum}

\begin{proof}
\begin{align}
 u^{(T)} = & ~ \alpha \cdot u^{(T-1)} + (1 - \alpha) \cdot h^{(T)} \notag \\
= & ~ \alpha \cdot u^{(T-1)} + (1-\alpha) \cdot ( W \cdot u^{(T-1)} + W_3 \cdot x^{(T-1)} + B ) \notag \\  
= & ~ (\alpha I + (1-\alpha) W) \cdot u^{(T-1)} + (1-\alpha) W_3 (  x^{(T-1)} + B) \notag \\
= & ~ (\alpha I + (1-\alpha) W)^2 \cdot u^{(T-2)}  + (\alpha I + (1-\alpha) W) (1-\alpha) W_3 (  x^{(T-2)} + B) \notag \\
& ~ + (1-\alpha) W_3 ( x^{(T-1)} + B) \notag \\
= & ~ (\alpha I + (1-\alpha) W)^{T-T_0} u^{(T_0)} + \sum_{t=T_0}^T (\alpha I + (1-\alpha) W)^{T-t-1} (1-\alpha) W_3 (  x^{(t)} + B) 
\end{align}
\end{proof}

\restate{cla:fru_ut_is_sum}

\begin{proof}
\begin{align}
& ~ u^{(T)} \notag \\
= & ~ u^{(T-1)} + \frac{1}{T} \cos(2\pi f T /T + \theta) h^{(T)} \notag \\
= & ~ u^{(T-1)} + \frac{1}{T} \cos(2\pi f T /T + \theta ) ( W u^{(T-1)} + W_3 x^{(T-1)} +B ) \notag \\
= & ~ (I+ \frac{1}{T} \cos(2\pi f T /T + \theta) W ) u^{(T-1)}  + \frac{1}{T} \cos(2\pi f T/T + \theta) (W_3 x^{(T-1)} + B) \notag \\
= & ~ (I+ \frac{1}{T} \cos(2\pi f T /T + \theta) W )  \cdot (u^{(T-2)} + \frac{1}{T} \cos (2\pi f (T-1)/ T + \theta) (W u^{(T-2)} + W_3 x^{(T-2)} + B) ) \notag \\
& ~ + \frac{1}{T} \cos(2\pi f T/T + \theta) (W_3 x^{(T-1)} + B) \notag \\
= & ~ (I + \frac{1}{T} \cos(2\pi f T/T + \theta) W) (I + \frac{1}{T} \cos(2\pi f (T-1)/T + \theta)) u^{(T-2)} \notag \\
& ~ + \frac{1}{T} \cos(2\pi f(T-1)/T + \theta) (W_3 x^{(T-2)} +B) \notag \\
& ~ + \frac{1}{T} \cos(2\pi f T/T + \theta) (W_3 x^{(T-1)} +B) \notag \\
= & ~\prod_{t=T_0+1}^T ( I + \frac{1}{T} \cos(2\pi f t/T + \theta) W) u^{(T_0)}  + \sum_{t=T_0+1}^T \frac{1}{T} \cos(2\pi f t / T + \theta) (W_3 x^{(t-1)} + B ).
\end{align}
\end{proof}

\restate{cla:fru_upper_bound_gradient}

\begin{proof}
Recall that Chain rule gives
\allowdisplaybreaks
\begin{align*}
\left\| \frac{\partial L}{ \partial u^{(T_0)} } \right\|= & ~ \left\| \left(\prod_{t=T_0+1}^T (I + \frac{1}{T} \cos ( 2\pi f t /T + \theta ) \cdot W ) \right)^\top \frac{\partial L}{\partial u^{(T)}} \right\| \\
\leq & ~ \sigma_{\max} \left( \prod_{t=T_0+1}^T (I + \frac{1}{T} \cos ( 2\pi f t /T + \theta ) \cdot W ) \right) \cdot \left\| \frac{\partial L}{\partial u^{(T)}} \right\| \\
= & ~\left( \prod_{t\in S_+} (1 + \frac{1}{T} \cos ( 2\pi f t /T + \theta ) \cdot  \sigma_{\max} (W) ) \right)  \\
 \cdot & ~  \left( \prod_{t \in S_-} (1 + \frac{1}{T} \cos ( 2\pi f t /T + \theta ) \cdot \sigma_{\min}( W) ) \right) \cdot \left\| \frac{\partial L}{\partial u^{(T)}} \right\|.
\end{align*}
The product term related to $S_-$ can be bounded by $1$ easily since $T > \sigma_{\max}(W) > \sigma_{\min}(W)$. Now the question is how to bound the product term related to $S_+$. Using the fact that $1+x \leq e^{x}$, then we have
\begin{align*}
 & ~ \prod_{t\in S_+} \left( 1 + \frac{1}{T} \cos ( 2\pi f t /T + \theta) \cdot  \sigma_{\max} (W) \right) \\
\leq & ~ \prod_{t \in S_+} \exp \left( \frac{1}{T} \cos ( 2\pi f t /T + \theta ) \cdot  \sigma_{\max} (W) \right) \\
= & ~ \exp \left( \sum_{t \in S_+} \frac{1}{T} \cos ( 2\pi f t /T + \theta) \cdot  \sigma_{\max} (W) \right) \\
\leq & ~ \exp{(\sigma_{\max}(W))}.
\end{align*}
Thus, putting it all together, we have
\begin{align*}
\left\| \frac{\partial L}{ \partial u^{(T_0)} } \right\| \leq e^{ \sigma_{\max} (W) } \left\| \frac{\partial L}{\partial u^{(T)}} \right\|.
\end{align*}
\end{proof}

In the proof of Claim~\ref{cla:fru_upper_bound_gradient} and Theorem~\ref{thm:fru_gradient}, the main property we used from Fourier basis is $|\cos(2\pi f t /T + \theta)| \leq 1, \forall t \in [0,T]$. Thus, it is interesting to understand how general is our proof technique. We can obtain the following result,
\begin{corollary}\label{thm:fru_gradient_general}
Let ${\cal F}$ denote a set. For each $f \in {\cal F}$, we define a function $y_f : \R \rightarrow \R$. For each $f \in {\cal F}$, $|y_f(t)|\leq 1, \forall t \in [0,T]$. We use $y_f(t)$ to replace $\cos(2\pi f t/T + \theta)$ in FRU update rule, i.e.,
\begin{align*}
h^{(t)} = & ~ W_1 W_2 \cdot u^{(t-1)} + W_2 b_1 + W_3 \cdot x^{(t-1)} + b_2, \\
u^{(t)} = & ~ u^{(t-1)} + \frac{1}{T} \cos (2\pi f t /T + \theta) \cdot h^{(t)}, & \mathrm{FRU} \\
u^{(t)} = & ~ u^{(t-1)} + \frac{1}{T} y_f(t) \cdot h^{(t)}. & \mathrm{Generalization~of~FRU}
\end{align*}
 Then
\begin{align*}
e^{- 2\sigma_{\max}(W_1W_2)} \left\| \frac{\partial L}{\partial u^{(T)}} \right\|
\leq \left\| \frac{\partial L}{ \partial u^{(T_0)} } \right\| \leq e^{ \sigma_{\max} (W_1W_2) } \left\| \frac{\partial L}{\partial u^{(T)}} \right\|.
\end{align*}
\end{corollary}
\begin{proof}
Let $W = W_1 W_1$. Recall that Chain rule gives
\begin{align*}
\left\| \frac{\partial L}{ \partial u^{(T_0)} } \right\|= & ~ \left\| \left(\prod_{t=T_0+1}^T (I + \frac{1}{T} \cdot y_f(t) \cdot W ) \right)^\top \frac{\partial L}{\partial u^{(T)}} \right\| \\
\leq & ~ \sigma_{\max} \left( \prod_{t=T_0+1}^T (I + \frac{1}{T} \cdot y_f(t) \cdot W ) \right) \cdot \left\| \frac{\partial L}{\partial u^{(T)}} \right\|,
\end{align*}
and
\begin{align*}
\left\| \frac{\partial L}{ \partial u^{(T_0)} } \right\|= & ~ \left\| \left(\prod_{t=T_0+1}^T (I + \frac{1}{T} \cdot y_f(t) \cdot W ) \right)^\top \frac{\partial L}{\partial u^{(T)}} \right\| \\
\geq & ~ \sigma_{\min} \left( \prod_{t=T_0+1}^T (I + \frac{1}{T} \cdot y_f(t) \cdot W ) \right) \cdot \left\| \frac{\partial L}{\partial u^{(T)}} \right\|.
\end{align*}
We define two sets $S_-$ and $S_+$ as follows 
\begin{align*}
S_- = \{ t ~|~ y_f(t) < 0, t = T_0+1, T_0+2, \cdots, T \}, \\
S_+ = \{ t ~|~ y_f(t) \geq 0, t = T_0+1, T_0+2, \cdots, T \}.
\end{align*}
Further, we 
\begin{align*}
& ~ \sigma_{\max} \left( \prod_{t=T_0+1}^T (I + \frac{1}{T} \cdot y_f(t) \cdot W ) \right) \\
\leq & ~ \left( \prod_{t\in S_+} ( 1 + \frac{1}{T} y_f(t) ) \cdot \sigma_{\max}(W) \right) \cdot \left( \prod_{t\in S_-} ( 1 + \frac{1}{T} y_f(t) ) \cdot \sigma_{\min}(W) \right) \\
\leq & ~ \left( \prod_{t\in S_+} ( 1 + \frac{1}{T} y_f(t) ) \cdot \sigma_{\max}(W) \right) \\
\leq & \exp( \sigma_{\max}(W)).
\end{align*}
Similarly as the proof of Theorem~\ref{thm:fru_gradient}, we have
\begin{align*}
\sigma_{\min} \left( \prod_{t=T_0+1}^T (I + \frac{1}{T} \cdot y_f(t) \cdot W ) \right) \geq \exp(-2\sigma_{\max}(W)).
\end{align*}
Therefore, we complete the proof.
\end{proof}

\begin{remark}
  Interestingly, our proof technique for Theorem~\ref{thm:fru_gradient} not only works for Fourier basis, but also works for any bounded basis. For example, the same statement(see Corollary~\ref{thm:fru_gradient_general} in Appendix) holds if we replace $\cos(2\pi f t/T + \theta)$ by $x_f(t)$ which satisfies $|x_f(t)| \leq O(1), \forall t \in [0,T]$.
\end{remark}

\section{More Experimental Results}
\label{sec:More_Experimental_Results}
We provide more experimental results on the MNIST dataset and detailed algorithms to generate 
synthetic data in this section.

\begin{figure*}[!h]
    \centering
    {\includegraphics[width=0.98\textwidth]{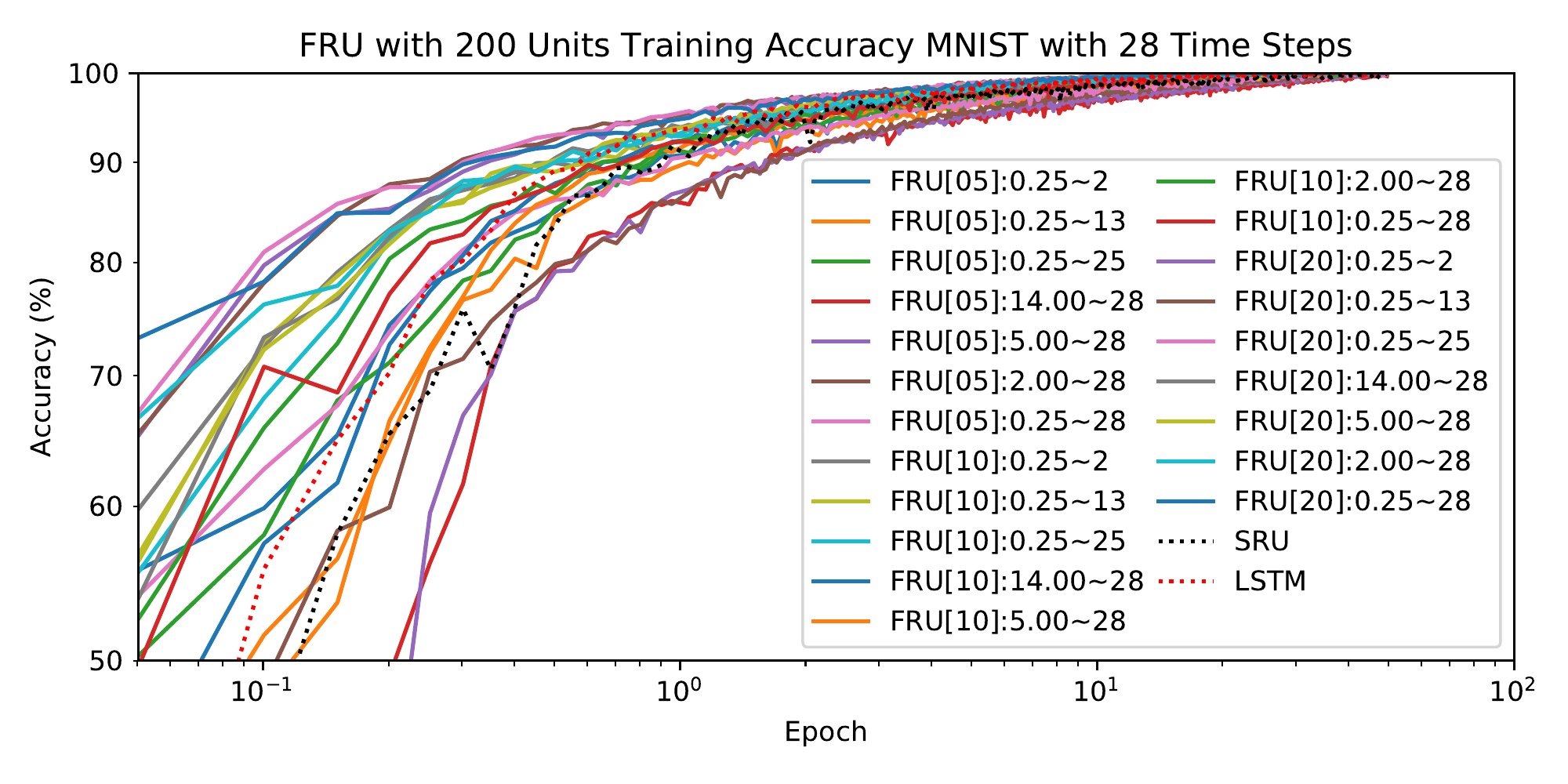}}\\ 
    {\includegraphics[width=0.98\textwidth]{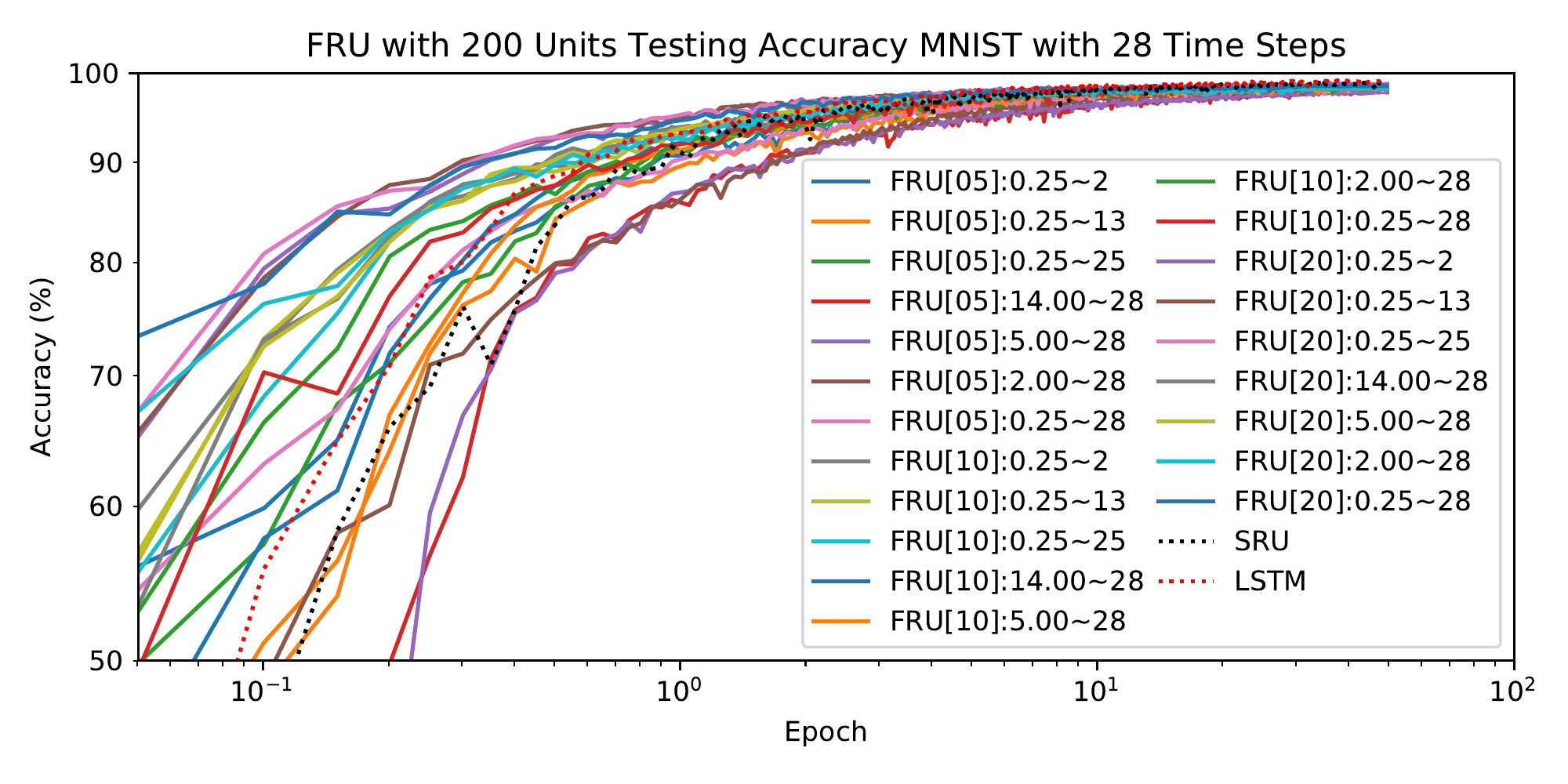}}
    \caption{Training and testing accuracy of different frequency combinations for FRU with 200 units on MNIST. Legend entry denotes id[\#freqs]:min freq$\sim$max freq. Frequencies are uniformly sampled in log space.}
    \label{fig:mnist_28_FRU_freqs}
\end{figure*}

\begin{figure*}[!h]
    \centering
    {\includegraphics[width=0.98\textwidth]{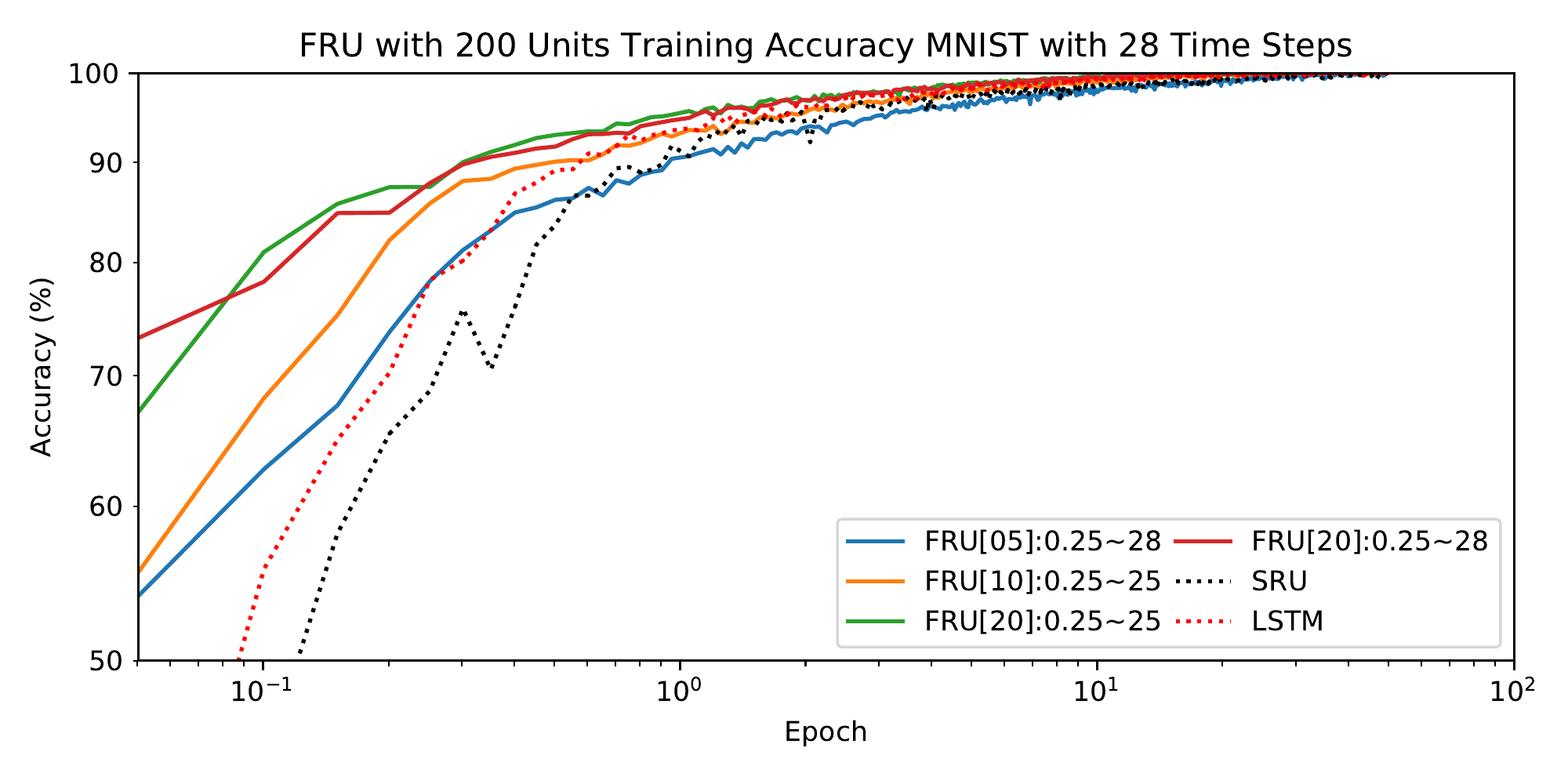}}\\
    {\includegraphics[width=0.98\textwidth]{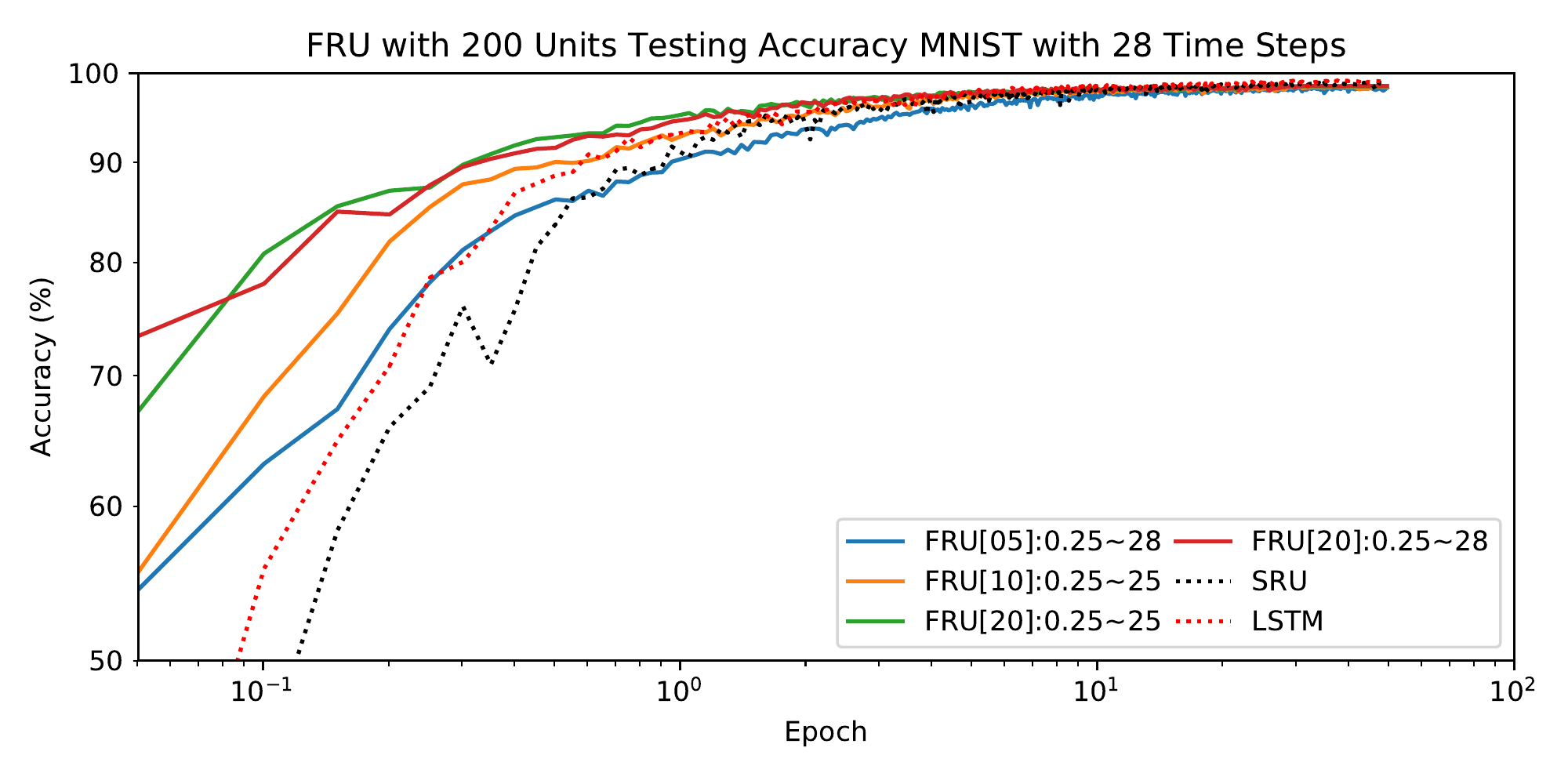}}
    \caption{Training and testing accuracy of different frequency combinations for FRU with 200 units on MNIST. Legend entry for FRU denotes FRU[\#freqs]:min freq$\sim$max freq. Frequencies are uniformly sampled in log space.}
    \label{fig:mnist_28_FRU_freqs_clean}
\end{figure*}

Fig.~\ref{fig:mnist_28_FRU_freqs} plots the training and testing accuracy for different frequencies of FRU on MNIST dataset with 28 time steps, respectively. 
In other words, the model reads one row of pixel at each time step. 
We stop at 50 epochs since most experiments has already converged. 
Given more frequencies, FRU in general achieves higher accuracy and faster convergence, 
while there are still discrepencies in accuracy for the different ranges of frequencies . 
Fig.~\ref{fig:mnist_28_FRU_freqs_clean} extracts the best ranges of frequencies given 5, 10, or 20 frequencies. 
We can see that frequency ranges that conver both low and high frequency components provide good performance. 

\begin{figure*}[!h]
    \centering
    {\includegraphics[width=0.98\textwidth]{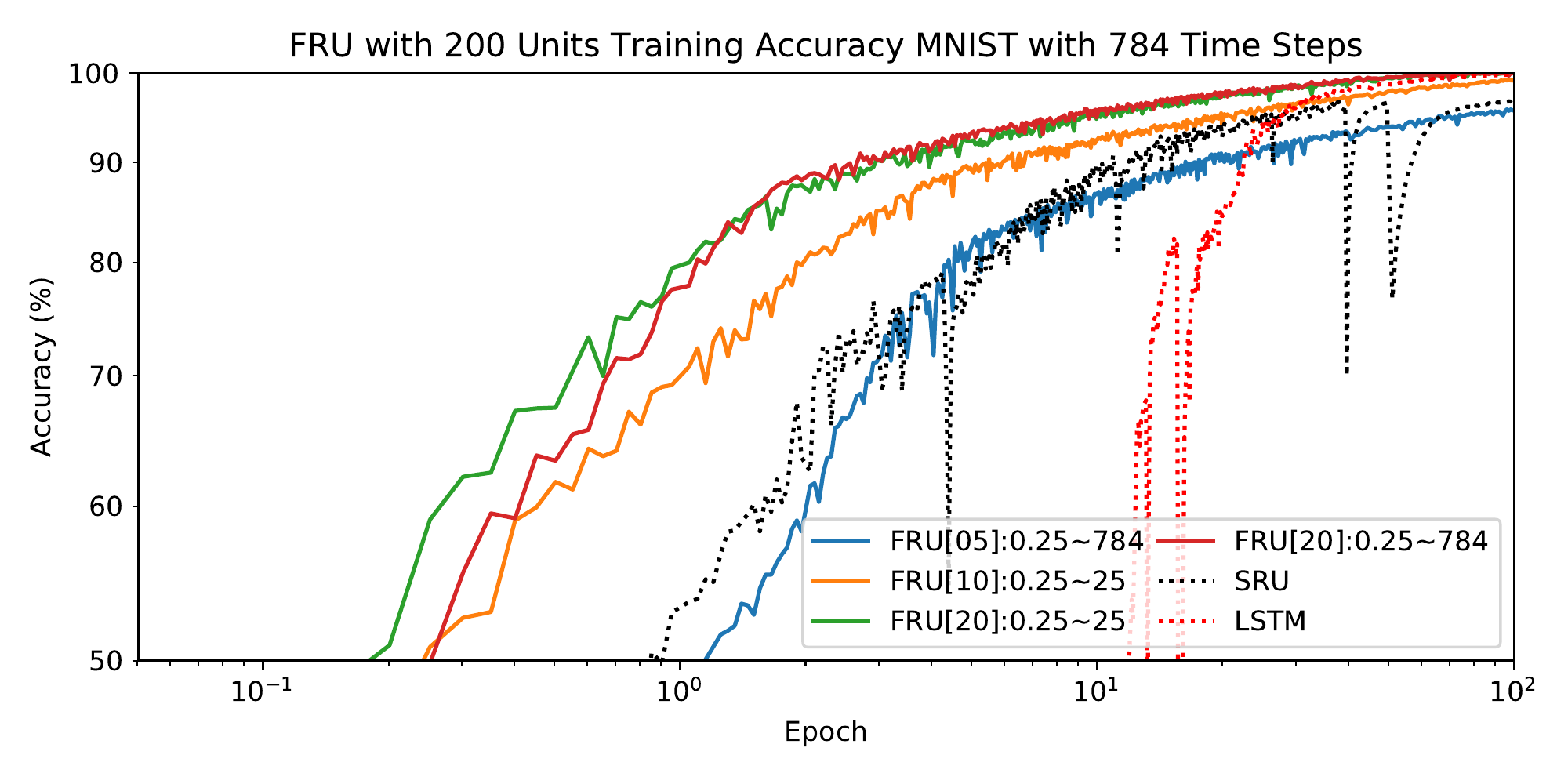}}\\
    {\includegraphics[width=0.98\textwidth]{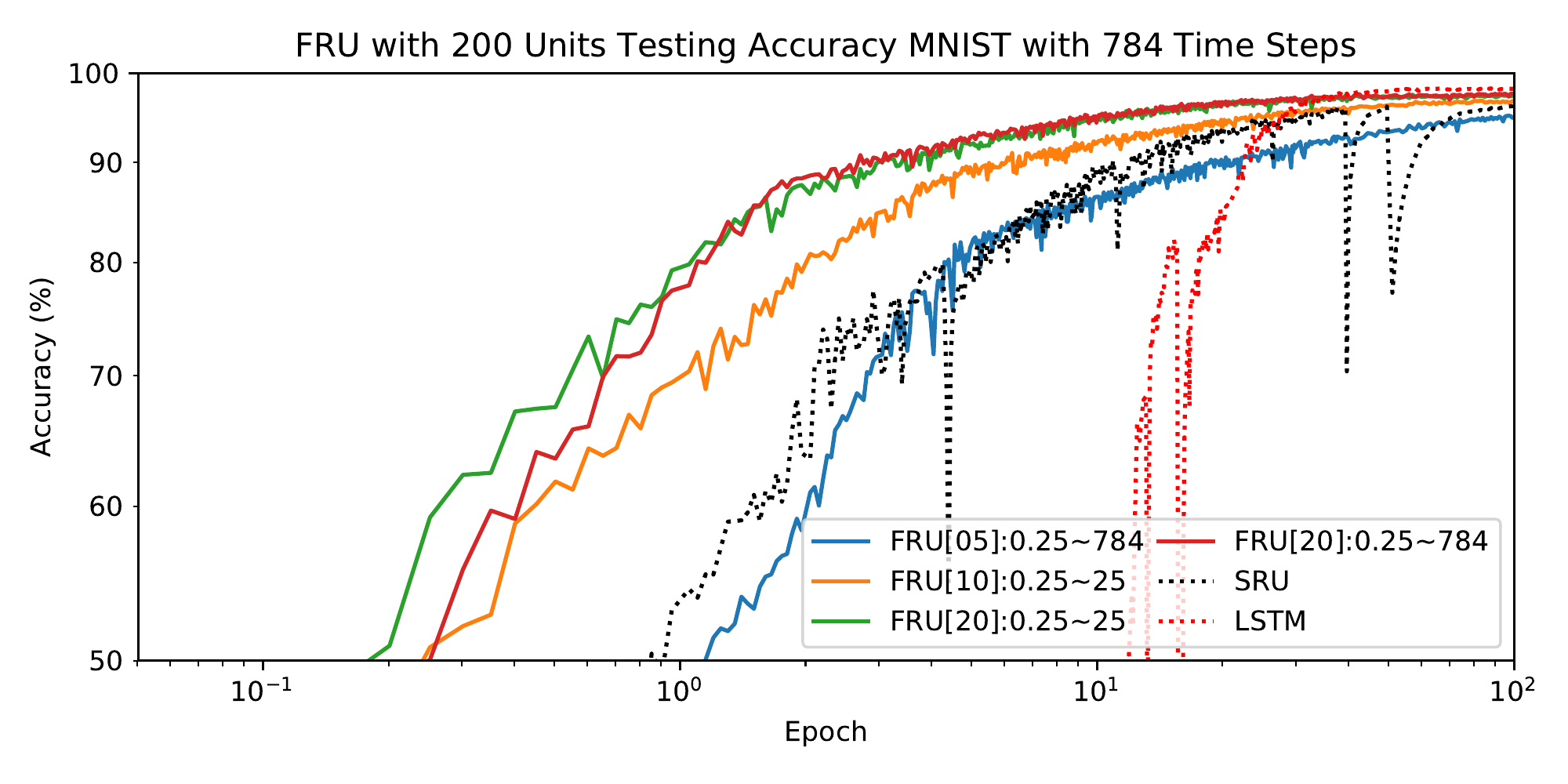}}
    \caption{Training and testing accuracy of different frequency combinations for FRU with 200 units on MNIST. Legend entry for FRU denotes FRU[\#freqs]:min freq$\sim$max freq. Frequencies are uniformly sampled in log space.}
    \label{fig:mnist_784_FRU_freqs_clean}
\end{figure*}

Fig.~\ref{fig:mnist_784_FRU_freqs_clean} shows the accuracy for different frequencies of FRU on MNIST dataset with 784 time steps. 
It is also observed that more frequencies provide higher accuracy and faster convergence, 
while frequency range of 0.25$\sim$25 gives similar accuracy to that of 0.25$\sim$784. 
This indicates that frequencies at relatively low ranges are already able to fit the data well.

\begin{algorithm}[!h]
\caption{Sin Synthetic Data for Fig.~\ref{fig:sin5_synthetic}}
\label{alg:sin5_synthetic}
\begin{algorithmic}[1]
\Procedure{\textsc{GenerateSinSyntheticData}}{$N,T,d$} \Comment{$N$ denotes the size, $T$ denotes the length, and $d$ denotes the number of frequencies} 

    \State $k \leftarrow 5$
    \For{$j=1 \to d$}
        \State Sample $f_j$ from $[0.1, 3]$ uniformly at random \Comment{generate frequencies}
        \State Sample $\theta_j$ from $[-1,1]$ uniformly at random \Comment{generate phases}
    \EndFor
    \For{$i=1 \to k$}
        \For{$j=1 \to d$}
            \State Sample $a_{i,j}$ from $[-1,1]$ uniformly at random \Comment{generate coefficients}
        \EndFor
    \EndFor
    \For {$l=1 \to N$ } 
        \For{$i = 1 \to k$}
            \State    Sample $\delta_i$ from $\mathcal{N}(0,0.1)$  \Comment{generate mixture rate}
            \State    Sample $b_i$ from $\mathcal{N}(0,0.1)$  \Comment{generate mixture bias}
        \EndFor 
    \For{$t=1 \to T$}
        \State $x_{l,t} \leftarrow \sum_{i=1}^k \delta_i \cdot \sum_{j=1}^{d} a_{i,j} \sin( \frac{ 2 \pi f_j (t-0.5T)}{0.5T} + 2 \pi \theta_j) + b_i$ \Comment{mix $k$ components}
    \EndFor
    \EndFor
    \State \Return $x$ \Comment{$x\in \R^{N\times T}$}
\EndProcedure
\end{algorithmic}
\end{algorithm}

\begin{algorithm}[!h]
\caption{Polynomial Synthetic Data for Fig.~\ref{fig:poly_synthetic}}
\label{alg:poly_synthetic}
\begin{algorithmic}[1]
\Procedure{\textsc{GeneratePolySyntheticData}}{$N,T,d$} \Comment{$N$ denotes the size, $T$ denotes the length, and $d$ denotes the degree of polynomial} 
    \State $k \leftarrow 5$
    \For{$i=1 \to k$}
        \For{$j=1 \to d$}
            \State Sample $a_{i,j}$ from $[-1,1]$ uniformly at random \Comment{generate coefficient}
        \EndFor
    \EndFor
    \For {$l = 1 \to N$ } 
            \For{$i =1 \to k$}
    \State    Sample $\delta_i$ from $\mathcal{N}(0,0.1)$ \Comment{generate mixture rate}
    \State    Sample $b_i$ from $\mathcal{N}(0,0.1)$ \Comment{generate mixture bias}
            \EndFor
    \For{$t=1 \to T$}
        \State $x_{l,t} \leftarrow \sum_{i=1}^k \delta_i \cdot \sum_{j=1}^{d} a_{i,j}(\frac{t-0.5T}{0.5T})^j + b_i$ \Comment{mix $k$ polynomials}
    \EndFor
    \EndFor
    \State \Return $x$ \Comment{$x\in \R^{N\times T}$}
\EndProcedure
\end{algorithmic}
\end{algorithm}

Alg.~\ref{alg:sin5_synthetic} and Alg.~\ref{alg:poly_synthetic} give the pseudocode for generation of mix-sin and mix-poly synthetic datasets. 
The procedures for mix-sin and mix-poly are similar, while both datasets are combinations of multiple sine/polynomial curves, respectively. 
For the mix-sin dataset, amplitudes, frequencies and phases of each curve are randomly generated. 
For the mix-poly dataset, amplitudes are randomly generated for each degree of each curve. 
In the mixing process, different curves are mixed with random mixture rates and biases.

\section{Acknowledgments}
The authors would like to thank Danqi Chen, Yu Feng, Rasmus Kyng, Eric Price, Tianxiao Shen, and Huacheng Yu for useful discussions. 









\end{document}